\documentclass[sigconf, authorversion]{acmart}
\AtBeginDocument{%
  }

\copyrightyear{2026}
\acmYear{2026}
\setcopyright{cc}
\setcctype{by-nc-nd}
\acmConference[KDD '26]{Proceedings of the 32nd ACM SIGKDD Conference on Knowledge Discovery and Data Mining V.2}{August 09--13, 2026}{Jeju Island, Republic of Korea}
\acmBooktitle{Proceedings of the 32nd ACM SIGKDD Conference on Knowledge Discovery and Data Mining V.2 (KDD '26), August 09--13, 2026, Jeju Island, Republic of Korea}
\acmDOI{10.1145/3770855.3817766}
\acmISBN{979-8-4007-2259-2/2026/08}

\usepackage{amsfonts}       
\usepackage{nicefrac}       

\usepackage{amsmath}
\usepackage{mathtools}
\usepackage{amsthm}

\theoremstyle{plain}
\newtheorem{theorem}{Theorem}

\newtheorem{lemma}[theorem]{Lemma}

\theoremstyle{definition}

\newtheorem{remark}{Remark}

\usepackage{bm}

\usepackage{algorithm}
\usepackage[noend]{algorithmic}

\newcommand{\tabref}[1]{Table~\ref{#1}}
\newcommand{\figref}[1]{Figure~\ref{#1}}

\newcommand{\lemref}[1]{Lemma~\ref{#1}}
\newcommand{\thmref}[1]{Theorem~\ref{#1}}

\newcommand{\secref}[1]{Section~\ref{#1}}
\newcommand{\appref}[1]{Appendix~\ref{#1}}
\newcommand{\algoref}[1]{Algorithm~\ref{#1}}

\DeclareMathOperator*{\argmax}{argmax}
\DeclareMathOperator*{\argmin}{argmin}

\newcommand{\R}[1]{\mathbb{R}^{#1}}
\newcommand{\N}[1]{\mathbb{Z}^{#1}_{>0}}

\newcommand{\ind}{\mathbf{1}}

\begin{document}

\title{Agile Online Model Selection: Resolving Adaptation Lag via Safeguarded Large Learning Rates}

\author{Kei Takemura}
\orcid{0000-0002-7984-1202}
\email{kei\_takemura@nec.com}
\affiliation{%
  \institution{NEC Corporation}
  \city{Tokyo}
  \country{Japan}
}

\author{Ryuta Matsuno}
\orcid{0000-0002-4543-2128}
\email{ryuta-matsuno@nec.com}
\affiliation{%
  \institution{NEC Corporation}
  \city{Tokyo}
  \country{Japan}
}

\author{Keita Sakuma}
\orcid{0009-0008-1425-4972}
\email{keita.skm@nec.com}
\affiliation{%
  \institution{NEC Corporation}
  \city{Tokyo}
  \country{Japan}
}



\begin{abstract}
Maintaining predictive accuracy in non-stationary environments requires online model selection to adapt autonomously to unknown distribution shifts. However, existing tuning-free algorithms face a fundamental trade-off between robustness and agility. Specifically, to ensure dynamic regret bounds, they must restrict learning rates to small constants (e.g., $O(1)$). This restriction inevitably causes significant adaptation lag during abrupt changes. To resolve this, we propose a novel optimistic online mirror descent that utilizes safeguarded large learning rates up to $\Theta(T)$, where $T$ is the number of rounds. Our key technical contribution is a post-hoc penalty mechanism that dynamically monitors unstable updates and excludes learning rates incurring excessive regret, eliminating the need for restrictive a priori constraints. We show that the cumulative penalty remains $O(\log T)$, allowing our algorithm to match near-optimal worst-case guarantees while achieving superior rates in benign cases. Empirical evaluations on three synthetic and eleven diverse real-world datasets demonstrate that our approach reduces the adaptation lag from hundreds of rounds to a few rounds, consistently outperforming tuning-free baselines.
\end{abstract}



\maketitle


\section{Introduction}\label{sec:introduction}

Modern stream data analysis requires accurate predictive models over long periods in environments where statistical properties change unpredictably~\citep{domingos2000,gama2010knowledge,gomes2019,krawczyk2017}.
In such non-stationary environments, distribution shift, encompassing phenomena such as covariate shift and concept drift, significantly degrades the predictive accuracy of static models~\citep{gama2014survey,ditzler2015learning}.
For instance, factors such as sensor degradation or abrupt changes in user preferences frequently cause distribution shifts~\citep{gama2014survey}.
Therefore, ensuring the ability to autonomously track distribution shifts is essential for reliable real-world applications.

In Online Model Selection (OMS)~\citep{kolter2007dwm,kolter2005addexp,street2001sea,wang03awe,deckert11bwe,brzezinski14aue2,gomes2017adaptive,oza2001online},\footnote{
The literature describes this problem setting using terms such as online ensemble learning~\citep{oza2001online}, streaming ensemble~\citep{street2001sea}, and additive expert ensemble~\citep{kolter2005addexp}.
We refer to this setting as OMS to emphasize the adaptive weighting mechanism in non-stationary environments.
} a learner aims to maximize predictive performance by dynamically identifying the optimal model from a set of candidate models.
The learner operates multiple predictive models with different properties throughout a repeated process.
In each round, the learner evaluates the accuracy that individual models have demonstrated on recent data.
Then, it assigns larger weights to models that perform better under current data trends.
Through these dynamic weight assignments, OMS algorithms seek to preserve predictive performance in non-stationary environments.
This adaptability makes OMS a crucial framework for enhancing the robustness of predictive performance against complex and fluctuating data streams.

Conventional OMS algorithms often require frequent hyperparameter tuning~\citep{elwell11learnppnse,kolter2005addexp,kolter2007dwm,Lu2017,zhao18condor,Matsubara2025}.
For example, Dynamic Weighted Majority (DWM)~\citep{kolter2007dwm} requires setting the decay rate (learning rate), which is a parameter determining how much the weight of each model is attenuated according to the loss.
The predictive performance of these algorithms strongly depends on the hyperparameters.
When the frequency or rate of distribution shifts is unknown, one must continue adjusting these hyperparameters through trial and error.
In real-world environments where optimal hyperparameters change over time, this tuning effort makes the automation of predictive model operations impractical.

Existing algorithms with theoretical guarantees have enabled the automatic adjustment of parameters in OMS~\citep{chen21impossible,luo2015achieving,jun17a,koolen15second,derooij14flipflop}.
Among them, Multi-scale Multiplicative-weight with Correction (MsMwC)~\citep{chen21impossible} achieves a state-of-the-art theoretical guarantee based on the Optimistic Online Mirror Descent (OOMD) framework.
This algorithm dynamically selects an appropriate learning rate from a multi-scale grid of learning rates based on data.
This mechanism eliminates the need for hyperparameter tuning, allowing the learner to track distribution shifts autonomously.

However, the theoretical analysis of OOMD inherently creates an unavoidable barrier to adaptation agility.
The existing analysis relies heavily on approximating the exponential weight updates with a quadratic function, which is a standard technique widely employed in the analysis of OOMD.
Intuitively, to ensure the validity of this approximation, the algorithm must strictly satisfy the condition $\eta \|\ell_t\|_{\infty} \le O(1)$,
where $\eta$ denotes the learning rate and $\ell_t$ represents the losses of candidate models observed at round $t$.
If an algorithm utilizes a large learning rate that violates this condition, the approximation error becomes uncontrollable, rendering the derivation of a meaningful theoretical guarantee impossible.
Consequently, previous studies are forced to restrict learning rates to conservative values ($O(1)$) to preserve theoretical consistency.

This theoretical requirement for conservative learning rates creates a significant adaptation lag in practice, particularly during abrupt distribution shifts.
While the algorithm should ideally shift weights rapidly to a new optimal model when a drift occurs, the $O(1)$ constraint limits the weight fluctuation per round to a constant factor.
As a result, increasing the weight of the best model from its lower bound ($O(1/T)$) to a dominant level ($\Omega(1)$) requires $\Omega(\log T)$ rounds,
where $T$ is the number of weight updates in OMS.
Resolving this delay while maintaining the robustness guaranteed by the theoretical analysis remains a critical challenge in OMS.

To resolve the trade-off between robustness and agility, we propose a novel OOMD that utilizes safeguarded large learning rates.
While existing tuning-free algorithms restrict learning rates to $O(1)$ to ensure the stability of weight updates, our approach employs a broad geometric grid of learning rates scaling up to $\Theta(T)$.
This expanded search space allows the algorithm to shift weights to the optimal model within a few rounds when distribution shifts occur,
thereby eliminating the significant adaptation lag observed in existing algorithms.

This aggressive adaptation is reconciled with theoretical robustness through a penalty-based exclusion rule that manages the risks associated with large learning rates.
Specifically, the algorithm detects instability when the product of the learning rate and the loss exceeds the theoretical threshold required for the approximation.
If the total penalty increases and unstable updates have significantly affected the predictive performance (i.e., regret), the utilization of the corresponding large learning rate is terminated.
We formally show that this exclusion rule is indispensable: without this safeguard, large learning rates lead to linear regret even in static environments, implying a failure to converge to the optimal model.

Notably, our safeguard mechanism enables the proposed algorithm to match existing near-optimal regret bounds in the worst case while achieving superior rates in benign cases.
Our analysis shows that the cumulative penalty remains bounded by a negligible $O(\log T)$ term, preserving the worst-case performance guarantee.
Furthermore, when the loss vector can be predicted with high accuracy, the algorithm safely utilizes optimal learning rates larger than constant order.
This capability facilitates agile tracking not only for abrupt distribution shifts but also for incremental shifts, as demonstrated by our analysis.
Consequently, the proposed algorithm drastically reduces the adaptation lag and effectively handles continuous distribution shifts without compromising robustness.
These results show that the safeguard mechanism successfully integrates theoretical robustness with the agility necessary for practical data streams.

Empirical evaluations on synthetic datasets quantitatively verify the resolution of adaptation lag.
In the abrupt drift scenario, the proposed algorithm completes the adaptation to the optimal model within a few rounds, whereas MsMwC requires hundreds of rounds to shift the weights.
This agile adaptation reduces the cumulative loss by approximately 80\% compared to MsMwC.
Moreover, our algorithm maintains superior performance even under incremental changes and data corruption, consistently outperforming MsMwC and other baselines.
These results demonstrate that the safeguard mechanism successfully reconciles agility in abrupt changes with robustness against noise.

Furthermore, experiments on diverse real-world data streams confirm the versatility of our approach.
The proposed algorithm achieved the lowest cumulative loss among all tuning-free algorithms across all eleven datasets tested, ranging from classification to regression tasks.
In addition to its predictive performance, the algorithm maintains computational efficiency suitable for large-scale applications.
While the safeguard mechanism expands the search space of learning rates, the average runtime remains approximately twice that of MsMwC, consistent with our theoretical analysis.
In terms of scalability, empirical results confirm that the computation time scales linearly with the number of experts.
Collectively, these results demonstrate that our approach successfully resolves the fundamental trade-off between theoretical robustness and practical agility.

\section{Preliminaries}\label{sec:preliminaries}

We begin by establishing the mathematical notation used throughout this paper.
Let $[K] := \{1, 2, \dots, K\}$ denote the set of $K$ experts.
For any vector $x \in \mathbb{R}^K$, we denote its $i$-th coordinate by $x(i)$.
We define the $(K-1)$-dimensional probability simplex as $\Delta_K := \{p \in \mathbb{R}^K \mid p(i) \ge 0, \sum_{i \in [K]} p(i) = 1\}$, which represents the space of the learner's prediction distributions.
For a convex function $\psi$, the Bregman divergence between two points $x, y$ is defined as $D_{\psi}(y, x) := \psi(y) - \psi(x) - \langle \nabla\psi(x), y - x \rangle$.
In this paper, we consider $\log\log T$ factors as constants as in previous studies~\citep{gaillard14second,luo2015achieving}.
Thus, $O$-notation hides the constants and the $\log\log T$ factors.

We address the OMS problem, where a learner sequentially updates a distribution over a dynamic set of candidate models.
To ensure theoretical clarity in the main text, we focus on a simplified setting where the number of experts (candidate models) $K$ remains fixed and the losses lie within $[0, 1]$.
This formulation is known as \textit{tracking the best expert}~\citep{Herbster2001tracking,cesa2006prediction,wei16tracking}.
We provide the generalized algorithm and analysis for the general setting in \appref{app:problem_setting}.
The learning protocol proceeds as follows: at each round $t \in [T]$, the learner first receives a prediction vector $m_t \in [0, 1]^K$.\footnote{
    While classical tracking the best expert problems do not include a prediction vector, we introduce $m_t$ to clarify optimistic online learning algorithms in recent studies~\citep{wei16tracking,chen21impossible}. Algorithms for the classical setting can be applied by setting $m_t = 0$.
}
Next, the learner chooses a distribution $p_t \in \Delta_K$.
Then, the learner observes the loss vector $\ell_t \in [0, 1]^K$ and suffers a loss $\langle \ell_t, p_t \rangle$.
The learner's goal is to minimize the dynamic regret defined below.

To evaluate the algorithm's performance in non-stationary environments where the optimal expert changes over time, we adopt \textit{dynamic regret}.
Comparing performance against a fixed comparator $u$ implicitly assumes a stationary environment.
However, in non-stationary environments, the best expert shifts over time.
Therefore, we define a sequence of comparators $u_{1:T} = (u_1, \dots, u_T)$ for each round $t \in [T]$ and formulate the difference from the learner's cumulative loss as follows:
\begin{equation*}
    R_T(u_{1:T}) = \sum_{t=1}^T \langle \ell_t, p_t - u_t \rangle.
\end{equation*}
We quantify the degree of non-stationarity in the comparator sequence using the \textit{path-length} $V(u_{1:T}) = \sum_{t \in [T]} \|u_t - u_{t-1}\|_1$, with the convention $u_0(i) = 0$ for all $i \in [K]$.
This metric allows us to evaluate how quickly the learner adapts to unknown changes.
Consequently, dynamic regret analysis serves as a critical measure to demonstrate both the theoretical robustness and the practical agility of the proposed algorithm.

The existing algorithm that achieves a state-of-the-art dynamic regret bound is MsMwC~\citep{chen21impossible}.
This algorithm incorporates a multi-scale learning rate grid and a second-order correction term into the OOMD framework.
Specifically, it executes the following two-step update at each round:\footnote{Strictly speaking, the algorithm achieving the dynamic regret bound employs a two-layer structure of MsMwC. For clarity, we describe the single-layer update rule here.}
\begin{align*}
    w_t &= \arg\min_{w \in \Omega_t} \langle m_t, w \rangle + D_{\psi}(w, w_t'), \\
    w_{t+1}' &= \arg\min_{w \in \Omega_t} \langle \ell_t + a_t, w \rangle + D_{\psi}(w, w_t'),
\end{align*}
where $\Omega_t \subseteq \Delta_K$ is the decision set, and $a_t$ is the correction term defined by $a_{t}(i) = 32\eta_{t}(i)(\ell_{t}(i) - m_{t}(i))^2$.
The algorithm uses a negative entropy function with a learning rate vector $\eta$ as the regularizer:
$\psi(w) = \sum_{i=1}^K \frac{1}{\eta(i)} w(i) \ln w(i)$.
This multi-scale design enables tuning-free online learning that autonomously adapts to unknown data statistics.

However, the existing analysis for MsMwC relies on a stringent constraint regarding the product of the learning rate and the prediction error for the loss vector.
To maintain analytical consistency, the algorithm requires the condition $32\eta(i)|\ell_{t}(i) - m_{t}(i)| \le 1$ to hold for all rounds.
When the prediction error $|\ell_{t}(i) - m_{t}(i)|$ is large, this constraint forces the learning rate $\eta_{t}(i)$ to be $O(1)$.
Consequently, MsMwC restricts the multiplicative weight update to a constant factor range per round.
While this ensures stability, it makes immediate weight shifting to a new best expert impossible when distribution shifts occur.
Specifically, increasing a weight from the lower bound ($O(1/T)$) to a dominant level $\Omega(1)$ under this constraint requires $\Omega(\log T)$ rounds.
Thus, the conservative learning rate constraint imposed by the existing analysis is the direct cause of the significant \textit{adaptation lag} observed in practice.


\section{Proposed Algorithm}\label{sec:proposed}

\begin{algorithm}[tb]
    \caption{OOMD with Safeguarded Large Learning Rates}
    \label{alg:proposed}
    \begin{algorithmic}[1]
        \REQUIRE Horizon $T \in \N{}$.
        \STATE Set $M = \lceil \log_2 T^2 \rceil$ and $\varepsilon = \frac{1}{KMT^3}$.
        \STATE Define learning rates $\eta(j) = \frac{2^j}{16T}$ for all $j \in [M]$.
        \STATE Initialize weights $w'_1(i, j) \propto \eta(j)^2$ such that $\sum_{i,j} w'_1(i,j) = 1$.
        \STATE Initialize penalties $L_1(j) = 0$ for all $j \in [M]$.
        \FOR{$t = 1, 2, \dots, T$}
            \STATE Receive prediction vector $m_t \in \R{K}$.
            \STATE Let $\tilde{m}_t(i,j) = m_t(i)$ for all $i \in [K]$ and $j \in [M]$.
            \STATE Construct active set $\mathcal{J}_t = \{ j \in [M] \mid L_t(j) \le 1 \}$.
            \STATE Compute $w_t \in \argmin_{w \in \Omega_t} \langle \tilde{m}_t, w \rangle + D_{\psi}(w, w'_t)$, where $\Omega_t = \{ w \in \Delta_{K \times M} \mid w(i,j) \ge \varepsilon\,(\forall j \in \mathcal{J}_t),\ w(i,j) = \varepsilon\,(\forall j \notin \mathcal{J}_t) \}$.
            \STATE Compute final prediction $p_t(i) \propto \sum_{j \in \mathcal{J}_t} w_t(i,j)$.
            \STATE Play $p_t$ and observe loss vector $\ell_t \in [0,1]^{K}$.
            \STATE Let $\tilde{\ell}_t(i,j) = \ell_t(i)$ for all $i \in [K]$ and $j \in [M]$.
            \STATE Compute $w'_{t+1} \in \argmin_{w \in \Omega_t} \langle \tilde{\ell}_t + a_t, w \rangle + D_{\psi}(w, w'_t)$, with correction $a_t(i,j) = 32\eta(j)(\ell_t(i) - m_t(i))^2$.
            \STATE Update penalties $L_{t+1}(j)$ according to \eqref{eq:penalty_update}.
        \ENDFOR
  \end{algorithmic}
\end{algorithm}

We propose a novel algorithm, \textit{OOMD with Safeguarded Large Learning Rates}, designed to resolve adaptation lag by safely leveraging large learning rates.
Our algorithm builds upon the OOMD framework, which utilizes loss predictions to reduce regret.
While our algorithm shares the OOMD foundation with MsMwC~\citep{chen21impossible}, it differs from MsMwC in three key aspects, each addressing a specific limitation that prevents existing analyses from exploiting large learning rates.
First, regarding the architecture, \citet{chen21impossible} employed a two-layer hierarchical structure, where a master algorithm manages multiple base algorithms running with different learning rates.
In contrast, we redefine the decision space as a product space of experts indexed by $i$ and learning rates indexed by $j$, treating each pair $(i, j)$ as a single virtual expert, and execute OOMD weight updates directly on the expanded probability simplex $\Delta_{K \times M}$, where $M = \lceil \log_2 T^2 \rceil$.
As we will discuss in Remark~\ref{rem:single-layer}, this single-layer design is theoretically indispensable for utilizing large learning rates.
Second, regarding the range of learning rates, MsMwC restricts learning rates to $O(1)$ via the a priori stability constraint $32 \eta(i) |\ell_t(i) - m_t(i)| \le 1$, whereas our algorithm employs a broad geometric grid scaling up to $\Theta(T)$ to enable rapid weight redistribution under distribution shifts.
Third, regarding the stability control, while MsMwC enforces stability through the aforementioned a priori restriction, we introduce a post-hoc penalty mechanism that monitors unstable updates and excludes learning rates only after they incur excessive regret.
The complete procedure is outlined in Algorithm~\ref{alg:proposed}.

The regret guarantees of \algoref{alg:proposed} are summarized as follows.
In the worst case, our algorithm matches the state-of-the-art switching regret bound of MsMwC up to constant factors (Theorem~\ref{thm:switching_regret}), despite utilizing learning rates up to $\Theta(T)$.
In benign environments where the cumulative prediction error is small, our algorithm achieves a strictly tighter regret bound (Theorem~\ref{thm:switching_regret_benign}), eliminating the $O(S \log T)$ overhead incurred by the $O(1)$ learning rate restriction of MsMwC.
Furthermore, to the best of our knowledge, our algorithm provides the first path-length regret bound that adapts to the cumulative squared prediction error (Theorem~\ref{thm:path_length}).
We also prove that the safeguard mechanism is indispensable: without it, large learning rates lead to linear regret even in static environments (Theorem~\ref{thm:safeguard_necessity}).
We defer the formal statements of these results to Section~\ref{sec:regret}.

We now elaborate on the algorithmic components corresponding to the three differences from MsMwC outlined above.
The learning rate grid is defined geometrically as $\eta(j) = \frac{2^j}{16T}$ for $j \in [M]$, providing $M$ candidates that span from $\Theta(1/T)$ up to $\Theta(T)$.
This extensive range allows the proposed algorithm to autonomously explore and select the optimal adaptation agility corresponding to the unknown drift dynamics, enabling rapid weight redistribution under non-stationary data streams.

To manage the risk associated with large learning rates, our algorithm introduces a post-hoc penalty term $L$ that monitors and accumulates the impact of unstable updates.
For each learning rate level $j$, the learner detects instability when the product of the learning rate and the prediction error exceeds a certain threshold.
Specifically, we update the penalty $L_{t+1}(j)$ as follows:
\begin{equation}\label{eq:penalty_update}
    L_{t+1}(j) = L_t(j) + \sum_{i \in [K]} \ind[32 \eta(j)|r_t(i)| > 1] w_t(i, j) r_t(i),
\end{equation}
where $\ind[\cdot]$ is the indicator function and $r_t(i) = \ell_t(i) - m_t(i)$ denotes the prediction error.
Note that the penalty term could be negative.
The penalty term serves as a safeguard, precisely recording the additional regret incurred by large learning rates and ensuring the regret bound remains within the same order as existing guarantees.
Crucially, this post-hoc management mechanism eliminates the need for the restrictive a priori constraints on learning rates required by previous algorithms, thereby enabling aggressive weight updates in non-stationary environments.

Based on this penalty, the proposed algorithm dynamically excludes learning rates that exceed a safety threshold from the active set $\mathcal{J}_t$.
At the beginning of each round, the learner evaluates the cumulative penalty $L_t(j)$ and constructs the set $\mathcal{J}_t$ using only the indices that satisfy the following condition:
\begin{equation*}
    \mathcal{J}_t = \{j \in [M] \mid L_t(j) \le 1\}.
\end{equation*}
Learning rates with a cumulative penalty exceeding 1 are considered to have caused an unacceptable increase in the regret bound and are removed from the update process.
This exclusion rule limits the regret growth caused by large learning rates to a negligible level.
Consequently, our algorithm successfully reconciles the practical need for agility with state-of-the-art theoretical guarantees.

The computational efficiency of the proposed algorithm relies on the optimization step for the weight update in lines 9 and 13.
While this problem lacks a closed-form solution due to the truncated simplex $\Omega_t$, the Karush-Kuhn-Tucker (KKT) conditions reduce the problem to finding a single scalar Lagrange multiplier associated with the simplex constraint.
Since the sum of the weights is a monotonically increasing function of this multiplier, we can determine its optimal value via a binary search.
This reduction transforms the potentially complex high-dimensional optimization into a simple root-finding problem, enabling the algorithm to handle the weight redistribution efficiently even with the expanded search space.

This optimization strategy ensures a per-round time complexity of $O(K \log^2 T)$, maintaining scalability even with the expanded parameter space.
Specifically, the binary search requires $O(\log T)$ iterations to achieve the precision sufficient to prevent regret degradation.
Within each iteration, evaluating the sum of weights involves scanning the $K \times M$ virtual experts, which costs $O(KM)$ operations.
Since the number of learning rate layers is logarithmic in the horizon (i.e., $M = O(\log T)$), the total cost amounts to $O(K \log^2 T)$.
This polylogarithmic dependence on $T$ indicates that the computational overhead is negligible compared to the linear dependence on $K$, making the algorithm suitable for real-world applications.


\section{Regret Analysis}\label{sec:regret}

In this section we provide our regret bounds of the proposed algorithm.
We defer the missing proofs to \appref{app:proof}.

We begin by establishing a fundamental lemma that bounds the regret over any time interval.
This lemma decomposes the regret into standard terms minimized by an appropriate learning rate and a cumulative penalty term corresponding to unstable updates.

\begin{lemma}
    \label{lem:general_bound}
    Let $r_t(i) = \ell_t(i) - m_t(i)$ and $q_t(j) = \sum_{i \in [K]} w'_t(i,j)$.
    Under the assumptions in \secref{sec:preliminaries}, \algoref{alg:proposed} satisfies the following bound for any interval $\{t_1, \dots, t_2\} \subseteq [T]$, any comparator $u \in \Delta_K$, and any learning rate index $j^* \in \bigcap_{t = t_1}^{t_2} \mathcal{J}_t$:
    \begin{align*}
        &\quad \sum_{t = t_1}^{t_2} \langle \ell_t, p_t - u \rangle \\
        &\le \frac{4\log(KMT)}{\eta(j^*)}
        + 32\eta(j^*) \sum_{t = t_1}^{t_2} \sum_{i \in [K]} u(i) r_t(i)^2 \\
        &\quad + \sum_{t = t_1}^{t_2}\sum_{j \in [M]} \frac{q_{t_1}(j) - q_{t_2+1}(j)}{\eta(j)}
        + \sum_{t = t_1}^{t_2} \sum_{j \in [M]} \Delta L_t(j) \\
        &\quad + O\left(\frac{\log(KMT) + \max_{t \in [T]} \|r_t\|_\infty}{T^2}\right),
    \end{align*}
    where $\Delta L_t(j) = \sum_{i \in [K]} \ind[32\eta(j)|r_t(i)| > 1] w_t(i,j) r_t(i)$ represents the instantaneous penalty incurred by unstable updates.
\end{lemma}
\begin{proof}[Proof Sketch]
    First, we bound the regret in the original decision space by the regret in the expanded decision space:
    \begin{align*}
        \sum_{t = t_1}^{t_2} \langle \ell_t, p_t - u \rangle 
        &\le \sum_{t = t_1}^{t_2} \langle \tilde{\ell}_t, w_t - \bar{u} \rangle + O(T^{-3}),
    \end{align*}
    where $\bar{u} \in \Omega_t$ denotes a virtual comparator in the expanded decision space.
    Here, we defer the definition of $\bar{u}$ to the formal proof for clarity.
    Applying a standard OOMD analysis (e.g., Lemma 15 of \citet{chen21impossible}) yields
    \begin{align*}
        \sum_{t = t_1}^{t_2} \langle \tilde{\ell}_t, w_t - \bar{u} \rangle
        &\le \sum_{t = t_1}^{t_2} (D_{\psi}(\bar{u}, w'_t) - D_{\psi}(\bar{u}, w'_{t+1}))
        + \sum_{t = t_1}^{t_2} A_t,
    \end{align*}
    where $A_t = \langle \tilde{\ell}_t - \tilde{m}_t + a_t, w_t - w'_{t+1} \rangle - D_{\psi}(w'_{t+1}, w_t) - \langle a_t, w_t - \bar{u} \rangle$.
    We will discuss the terms on the right-hand side separately.

    The first term is bounded via straightforward algebraic manipulation as follows:
    \begin{align*}
        &\quad \sum_{t = t_1}^{t_2} (D_{\psi}(\bar{u}, w'_t) - D_{\psi}(\bar{u}, w'_{t+1})) \\
        &\le \frac{4\log(KMT)}{\eta(j^*)}
        + \sum_{j \in [M]} \frac{q_{t_1}(j) - q_{t_2+1}(j)}{\eta(j)}
        + O(\log(KMT)/T^2).
    \end{align*}

    The second term involves our novel analysis.
    We fix $t \in [T]$ arbitrarily.
    Let $v_t = \tilde{\ell}_t - \tilde{m}_t + a_t$ and
    \begin{align*}
        w^* \in \argmax_{w \in \R{K \times M}_+} \langle v_t, w_t - w \rangle - D_{\psi}(w, w_t),
    \end{align*}
    where $\R{K \times M}_+ = \{ x \in \R{K \times M} \mid \forall (i,j) \in [K] \times [M],\ x(i,j) \ge 0 \}$.
    Then, we have
    \begin{align*}
        &\quad A_t + \langle a_t, w_t - \bar{u} \rangle \\
        &= \langle v_t, w_t - w'_{t+1} \rangle - D_{\psi}(w'_{t+1}, w_t) \\
        &\le \langle v_t, w_t - w^* \rangle - D_{\psi}(w^*, w_t) \\
        &= \langle \nabla\psi(w_t) - \nabla\psi(w^*), w_t - w^* \rangle - D_{\psi}(w^*, w_t) \\
        &= \sum_{i \in [K],\, j \in [M]} \frac{w_t(i,j)}{\eta(j)} \left( \eta(j) v_t(i,j) - 1 + e^{-\eta(j) v_t(i,j)} \right).
    \end{align*}
    We will fix $(i,j)$ arbitrarily and discuss two cases based on the magnitude of the update.

    \paragraph{Case 1: small learning rate ($32\eta(j)|r_t(i)| \le 1$)}
    In this regime, the existing analysis holds.
    Specifically, we have
    \begin{align*}
        \eta(j)|v_t(i,j)| \le (1 + 32\eta(j)|r_t(i)|)\eta(j)|r_t(i)| \le 1.
    \end{align*}
    Therefore, since $e^{-x} - 1 + x \le x^2$ for $x \ge -1$, we obtain
    \begin{align*}
        &\quad \frac{w_t(i,j)}{\eta(j)} \left( \eta(j) v_t(i,j) - 1 + e^{-\eta(j) v_t(i,j)} \right) \\
        &\le \eta(j) w_t(i,j) v_t(i,j)^2
        \le 4 \eta(j) w_t(i,j) r_t(i)^2.
    \end{align*}

    \paragraph{Case 2: large learning rate ($32\eta(j)|r_t(i)| > 1$)}
    In this case, the condition implies
    \begin{align*}
        v_t(i,j) \ge (32\eta(j)|r_t(i)| - 1)|r_t(i)| > 0.
    \end{align*}
    Thus, we have
    \begin{align*}
        &\quad \frac{w_t(i,j)}{\eta(j)} \left( \eta(j) v_t(i,j) - 1 + e^{-\eta(j) v_t(i,j)} \right) 
        \le \frac{w_t(i,j)}{\eta(j)} \eta(j) v_t(i,j) \\
        &= 32 \eta(j) w_t(i,j) r_t(i)^2 + w_t(i,j) r_t(i).
    \end{align*}

    Summing over all $(i,j)$ and combining these two cases, we obtain
    \begin{align*}
        A_t
        &\le \sum_{i \in [K],\, j \in [M]} (32 \eta(j) w_t(i,j) r_t(i)^2
        - \langle a_t, w_t - \bar{u} \rangle) \\
        &\quad + \sum_{i \in [K],\, j \in [M]} \ind[32\eta(j)|r_t(i)| > 1] w_t(i,j) r_t(i) \\
        &\le 32\eta(j^*) \sum_{i \in [K]} u(i)(\ell_t(i) - m_t(i))^2
        + \sum_{j \in [M]} \Delta L_t(j) \\
        &\quad + O(\|r_t\|_\infty / T^{2}).
    \end{align*}
\end{proof}

\begin{remark}[Necessity of single-layer architecture over \citet{chen21impossible}]\label{rem:single-layer}
Our proof of \lemref{lem:general_bound} shows the theoretical necessity of our single-layer architecture over the two-layer architecture employed in \citet{chen21impossible}.
The existing analysis for two-layer MsMwC relies heavily on a negative term derived from the correction $-\langle a_t, w_t - \bar{u} \rangle$ within the base algorithm's regret bound.
This term possesses a coefficient large enough to offset the regret incurred by the master algorithm, thereby establishing the desired regret bound.
However, as demonstrated in Case 2, the analysis for large learning rates does not yield this negative cancellation term.
Consequently, a two-layer structure would fail to achieve the same regret bound when a large learning rate is selected.
While the two-layer MsMwC can be converted into a single-layer architecture as discussed in \citet{chen21impossible},
such a conversion within the existing analysis remains restricted to small learning rates.
Our novel analysis overcomes this limitation, enabling the algorithm to safely leverage large learning rates while maintaining robust theoretical guarantees.
\end{remark}

The significance of \lemref{lem:general_bound} lies in the explicit inclusion of the cumulative penalty $L$ within the regret bound.
This property implies that the regret increase caused by unstable weight updates (due to large learning rates) can be managed post-hoc within the analysis.
This serves as the theoretical foundation for proving that our aggressive adaptation strategy does not compromise theoretical performance guarantees.

Based on this lemma, we first show that the proposed algorithm maintains a near-optimal regret bound in the worst case, where the optimal expert switches $S$ times.

\begin{theorem}[Switching Regret for Worst Case]
    \label{thm:switching_regret}
    Let $\mathcal{S} = \{\mathcal{I}_1, \dots, \mathcal{I}_S\}$ be any partition of $[T]$ and $u_1, \dots, u_S \in \Delta_K$ be any sequence of comparators.
    Let $Q_s = \sum_{t \in \mathcal{I}_s} \sum_{i \in [K]} u_s(i) r_t(i)^2$ denote the cumulative prediction error for the comparator in interval $s$.
    Assume the same as in \lemref{lem:general_bound}.
    Then, \algoref{alg:proposed} satisfies:
    \begin{equation}
        \sum_{s \in [S]} \sum_{t \in \mathcal{I}_s} \langle \ell_t, p_t - u_s \rangle
        \le O\left( S \log(KT) + \sum_{s \in [S]} \sqrt{\log(KT) Q_s} \right).
    \end{equation}
\end{theorem}

This bound is derived by summing the result of Lemma~\ref{lem:general_bound} over $S$ intervals.
Crucially, our safeguard mechanism ensures that the total accumulated penalty $\sum_{t, j} \Delta L_t(j)$ remains bounded by $O(\log T)$ throughout the entire horizon.
This fact shows that allowing learning rates up to $\Theta(T)$ with our safeguard mechanism does not compromise the worst-case regret guarantees.

Furthermore, in \textit{benign cases} where the cumulative prediction error $Q_s$ is small, our aggressive approach achieves a superior regret bound compared to existing regret bounds.

\begin{theorem}[Switching Regret for Benign Case]
    \label{thm:switching_regret_benign}
    Under the same assumptions as \thmref{thm:switching_regret}, further assume that for each interval $s$, there exists an index $j^*_s$ such that
    $\eta(j^*_s) = \Theta\left( \sqrt{\frac{\log(KMT)}{Q_s}} \right)$
    and $j^*_s \in \mathcal{J}_t$ for all $t \in \mathcal{I}_s$.
    Then, \algoref{alg:proposed} satisfies:
    \begin{equation*}
        \sum_{s \in [S]} \sum_{t \in \mathcal{I}_s} \langle \ell_t, p_t - u_s \rangle
        \le O\left( \log T + \sum_{s \in [S]} \sqrt{\log(KT) Q_s} \right).
    \end{equation*}
\end{theorem}

We compare our bound with that of MsMwC~\citep{chen21impossible}, which achieves $O(S \log(KT) + \sum_{s} \sqrt{\log(KT) Q_s})$ in the same setting.
The first term $O(S \log T)$ arises from the $O(1)$ restriction on learning rates.
Specifically, in benign intervals where $Q_s$ is small, the optimal learning rate $\eta^*_s = \Theta(\sqrt{\log(KMT) / Q_s})$ exceeds the constant order, forcing MsMwC to incur an $O(\log T)$ overhead per interval.
Our algorithm permits learning rates up to $\Theta(T)$ and selects $\eta^*_s$ in benign intervals, eliminating this overhead while keeping the total accumulated penalty bounded by $O(\log T)$ globally.
Consequently, our bound is strictly tighter than that of MsMwC when $S$ is large and each $Q_s$ is small.

The proposed algorithm also exhibits high adaptability to continuous changes in the comparator, quantified by the path-length.
\begin{theorem}[Path-Length Regret Bound]
    \label{thm:path_length}
    Let the path-length of the comparator sequence be $V(u_{1:T}) = \sum_{t \in [T]} \|u_t - u_{t-1}\|_1$.
    Let $Q(u_{1:T}) = \sum_{t \in [T]} \sum_{i \in [K]} u_t(i) r_t(i)^2$ denote the cumulative squared prediction error of the comparator sequence.
    Assume the same as in \lemref{lem:general_bound}.
    Then, \algoref{alg:proposed} ensures:
    \begin{equation*}
        R_T(u_{1:T}) \le O\left( \sqrt{V(u_{1:T}) Q(u_{1:T}) \log(KT)} + V(u_{1:T}) \log(KT) \right).
    \end{equation*}
\end{theorem}

\begin{remark}[Terminology and comparison with prior path-length bounds]
Our bound is ``second-order'' in the sense of being adaptive to the squared prediction errors $Q(u_{1:T})$, following the convention of variance adaptive bounds in the expert problems~\citep{gaillard14second}.
We note that this is distinct from the second-order path variations recently studied in online convex optimization (OCO)~\citep{jacobsen2024equivalence,baby2023second}.
Moreover, these works address OCO over general convex sets and incur polynomial dependence on the dimension, whereas our algorithm exploits the simplex geometry to achieve a logarithmic dependence on the number of experts $K$.
We are not aware of prior path-length bounds for tracking the best expert that simultaneously achieve variance-adaptivity in the prediction error.
\end{remark}

This theorem suggests that our safeguard mechanism enables agile tracking not only for abrupt concept drifts but also for incremental changes, by maintaining access to large learning rates whenever they are safe to use.
Notably, analogous to \thmref{thm:switching_regret_benign}, this bound improves significantly in benign cases where large learning rates are active.
When the algorithm can safely select a large learning rate ($\eta = \Theta(\sqrt{V(u_{1:T}) \log(KT) / Q(u_{1:T})})$), the proposed algorithm achieves $O(\sqrt{V(u_{1:T}) Q(u_{1:T}) \log(KT)} + \log T)$ dynamic regret.
This indicates that our algorithm can achieve near-zero dynamic regret in favorable environments, resolving the adaptation lag even for continuously drifting distributions.

Finally, we demonstrate that our safeguard mechanism is indispensable for achieving these results.
We prove that naively introducing large learning rates without the penalty-based exclusion leads to linear regret even in static environments.

\begin{theorem}[Necessity of Safeguard]
    \label{thm:safeguard_necessity}
    Consider an algorithm that employs the same learning rate grid as \algoref{alg:proposed} but does not exclude learning rates based on the penalty $L$ (i.e., $\mathcal{J}_t = [M]$ always).
    There exists a sequence of losses and predictions such that this algorithm suffers $\Omega(T)$ regret, whereas \algoref{alg:proposed} achieves $O(\sqrt{T})$.
\end{theorem}

\begin{proof}[Proof Sketch]
    Consider a setting with two experts where expert 1 has constant loss $1$ and the other has loss $1/2$.
    The prediction vector $m$ is misleading: $m(1) = 0$ and $m(2) = 1/2$.
    An algorithm without the safeguard using a learning rate $\eta = \Theta(T)$ will aggressively update weights based on the misleading prediction $m$ before observing the true loss.
    Due to the large learning rate, the weight on the optimal expert is driven to near zero in the optimistic update step.
    Consequently, the learner places $\Omega(1)$ probability on the suboptimal expert at every round, incurring linear regret.
    Our safeguard detects this instability (large learning rates and large prediction error) immediately and excludes such dangerous learning rates, preserving the regret bound.
\end{proof}

\begin{table*}[t]
\centering
\caption{
Cumulative losses for Rotated MNIST dataset under abrupt drift, incremental drift, and corruption scenarios (mean $\pm$ standard error). The best results are highlighted in bold.
The proposed algorithm consistently achieves the lowest cumulative loss across all drift types, including the abrupt drift scenario.
}
\label{tab:synthetic_results}
\begin{tabular}{lccccc}
\toprule
 & ATV & Squint & CBCE & MsMwC & Proposed \\
\midrule
Abrupt Drift & 114.77 $\pm$ 0.30 & 165.84 $\pm$ 0.48 & 47.17 $\pm$ 0.09 & 254.85 $\pm$ 0.84 & \textbf{44.48 $\pm$ 0.10} \\
Incremental Drift & 141.66 $\pm$ 0.06 & 206.04 $\pm$ 0.06 & 53.86 $\pm$ 0.09 & 282.12 $\pm$ 0.05 & \textbf{42.97 $\pm$ 0.09} \\
Corruption & 101.69 $\pm$ 0.09 & 137.77 $\pm$ 0.09 & 68.67 $\pm$ 0.11 & 158.09 $\pm$ 0.09 & \textbf{67.45 $\pm$ 0.12} \\
\bottomrule
\end{tabular}
\end{table*}

\begin{figure*}[tb]
    \centering
    \includegraphics[width=0.9\textwidth]{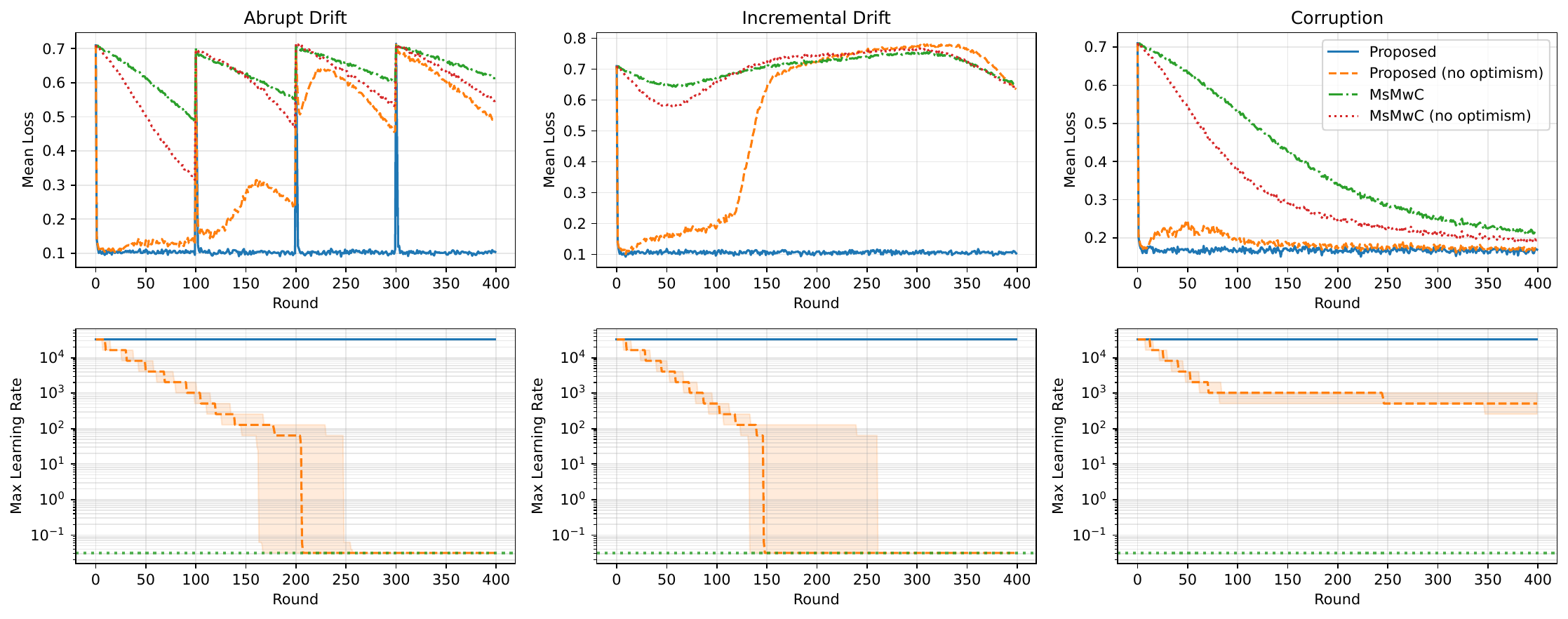}
    \Description{A grid of six line plots arranged in two rows and three columns over rounds 0 to 400. Columns correspond to abrupt drift, incremental drift, and corruption scenarios. The top row plots mean loss for four variants: Proposed, Proposed (no optimism), MsMwC, and MsMwC (no optimism). The Proposed variant stays near 0.1 throughout in all scenarios. Under abrupt drift its loss spikes briefly at each drift point near rounds 100, 200, and 300, then immediately returns to the baseline, whereas the no-optimism variants and MsMwC variants recover much more slowly or fail to recover. Under incremental drift the Proposed variant remains flat near 0.1 while all other variants drift upward toward 0.7. Under corruption the Proposed variant stays low while the others decay slowly from about 0.7 toward 0.2. The bottom row plots the median maximum learning rate on a log scale from 0.1 to 10000, with shaded 25th to 75th percentile bands. The Proposed variant maintains a constant large learning rate near 10000, while the no-optimism variant decays in stair-step fashion down to roughly 0.1 (or plateaus around 1000 under corruption).}
    \caption{Ablation study on the Rotated MNIST dataset under abrupt, incremental, and corruption drift scenarios. The top row shows the trajectory of the mean loss; standard deviations are omitted for visual clarity. The bottom row illustrates the median of the maximum learning rate in the active set, where the shaded regions indicate the 25th–75th percentiles.}
    \label{fig:ablation_study}
\end{figure*}

\begin{figure}[tb]
    \centering
    \includegraphics[width=0.4\textwidth]{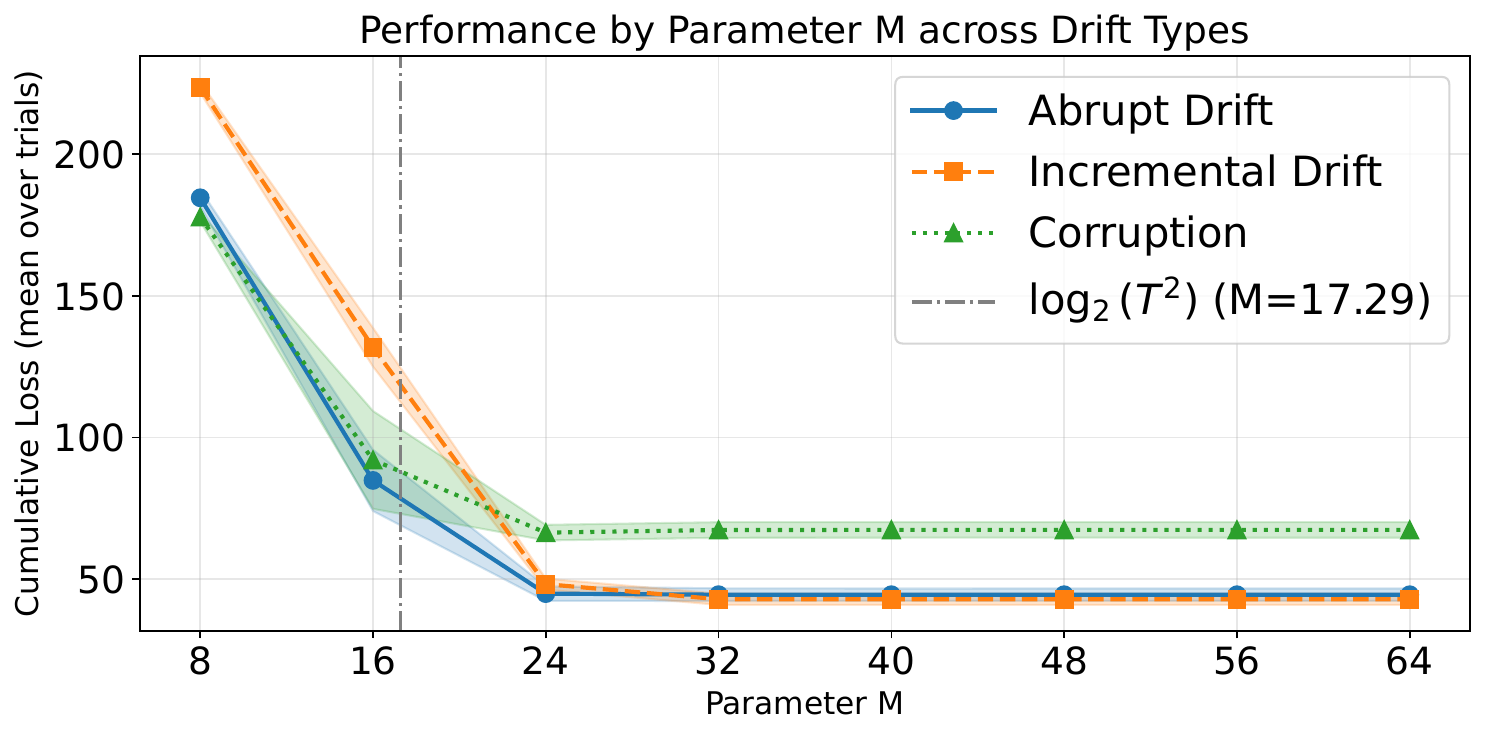}
    \Description{A line plot of cumulative loss (mean over trials) versus the parameter M (number of learning rate candidates), with M ranging from 8 to 64 on the horizontal axis. Three curves correspond to abrupt drift (blue circles), incremental drift (orange squares), and corruption (green triangles), with shaded standard deviation bands. A vertical grey dash-dotted line marks the theoretical lower bound at M = 17.29. All three curves drop sharply between M=8 and M=24, then plateau: abrupt and incremental settle near a cumulative loss of about 45, and corruption plateaus near 65. Performance does not degrade as M increases beyond the theoretical bound.}
    \caption{
    Sensitivity analysis of the cumulative loss with respect to the number $M$ of learning rate candidates across different drift scenarios. The vertical dash-dotted line indicates the theoretical requirement $M = \lceil \log_2 T^2 \rceil$. The shaded regions represent the standard deviation.
    The algorithm exhibits no performance degradation even when $M$ exceeds the theoretical requirement, demonstrating robustness to hyperparameter selection.
    }
    \label{fig:sensitivity}
\end{figure}

\begin{table*}[t]
\centering
\caption{
Normalized cumulative losses on diverse real-world datasets (mean ± standard error). Losses are normalized such that the average loss of the MsMwC baseline equals 100. The classification task and regression task are represented by (C) and (R), respectively. The best results are highlighted in bold.
The proposed algorithm outperforms or matches all theoretically guaranteed baselines across every dataset, demonstrating its superior versatility and robustness in practical, non-stationary environments.
}
\label{tab:real_world_results}
\begin{tabular}{lccccc}
\toprule
 & ATV & Squint & CBCE & MsMwC & Proposed \\
\midrule
Airlines (C) & 96.99 $\pm$ 0.46 & 99.64 $\pm$ 0.51 & 91.18 $\pm$ 0.39 & 100.00 $\pm$ 0.52 & \textbf{88.24 $\pm$ 0.33} \\
arXiv (C) & 96.30 $\pm$ 0.12 & 98.67 $\pm$ 0.11 & \textbf{92.55 $\pm$ 0.18} & 100.00 $\pm$ 0.14 & \textbf{92.55 $\pm$ 0.18} \\
Electricity (C) & 89.92 $\pm$ 0.92 & 97.04 $\pm$ 0.97 & 84.34 $\pm$ 0.90 & 100.00 $\pm$ 1.02 & \textbf{74.87 $\pm$ 0.80} \\
Forest (C) & 82.08 $\pm$ 0.82 & 94.68 $\pm$ 0.89 & 60.58 $\pm$ 0.57 & 100.00 $\pm$ 0.92 & \textbf{53.94 $\pm$ 0.54} \\
HuffPost (C) & 88.74 $\pm$ 0.13 & 97.46 $\pm$ 0.20 & 77.77 $\pm$ 0.19 & 100.00 $\pm$ 0.27 & \textbf{77.08 $\pm$ 0.20} \\
Powersupply (C) & 97.94 $\pm$ 0.11 & 99.69 $\pm$ 0.10 & 95.13 $\pm$ 0.11 & 100.00 $\pm$ 0.10 & \textbf{93.13 $\pm$ 0.11} \\
Rialto (C) & 82.86 $\pm$ 0.21 & 90.71 $\pm$ 0.21 & 70.14 $\pm$ 0.16 & 100.00 $\pm$ 0.23 & \textbf{61.62 $\pm$ 0.15} \\
Weather (C) & 95.87 $\pm$ 0.18 & 99.40 $\pm$ 0.26 & 93.45 $\pm$ 0.15 & 100.00 $\pm$ 0.28 & \textbf{93.19 $\pm$ 0.18} \\
Yearbook (C) & 45.58 $\pm$ 0.22 & 70.88 $\pm$ 0.44 & 29.21 $\pm$ 0.21 & 100.00 $\pm$ 0.69 & \textbf{26.05 $\pm$ 0.28} \\
Bikesharing (R) & 74.10 $\pm$ 0.79 & 84.08 $\pm$ 0.85 & 61.27 $\pm$ 0.64 & 100.00 $\pm$ 1.04 & \textbf{58.17 $\pm$ 0.79} \\
Temperature (R) & 83.95 $\pm$ 0.47 & 88.83 $\pm$ 0.33 & 77.75 $\pm$ 0.39 & 100.00 $\pm$ 0.87 & \textbf{75.77 $\pm$ 0.38} \\
\bottomrule
\end{tabular}
\end{table*}

\begin{table}[t]
\centering
\caption{
Average computation time per round (milliseconds) on the Rotated MNIST dataset (mean $\pm$ standard error).
The proposed algorithm operates within 30 ms, maintaining practical speed for online stream processing despite the computational overhead compared to MsMwC.
}
\label{tab:computation_time}
\begin{tabular}{lc}
\toprule
Algorithm & Time (ms) \\
\midrule
ATV & 10.73 $\pm$ 0.01 \\
Squint & 28.45 $\pm$ 0.03 \\
CBCE & 5.60 $\pm$ 0.01 \\
MsMwC & 14.04 $\pm$ 0.01 \\
Proposed & 29.77 $\pm$ 0.02 \\
\bottomrule
\end{tabular}
\end{table}

\begin{figure}[tb]
    \centering
    \includegraphics[width=0.4\textwidth]{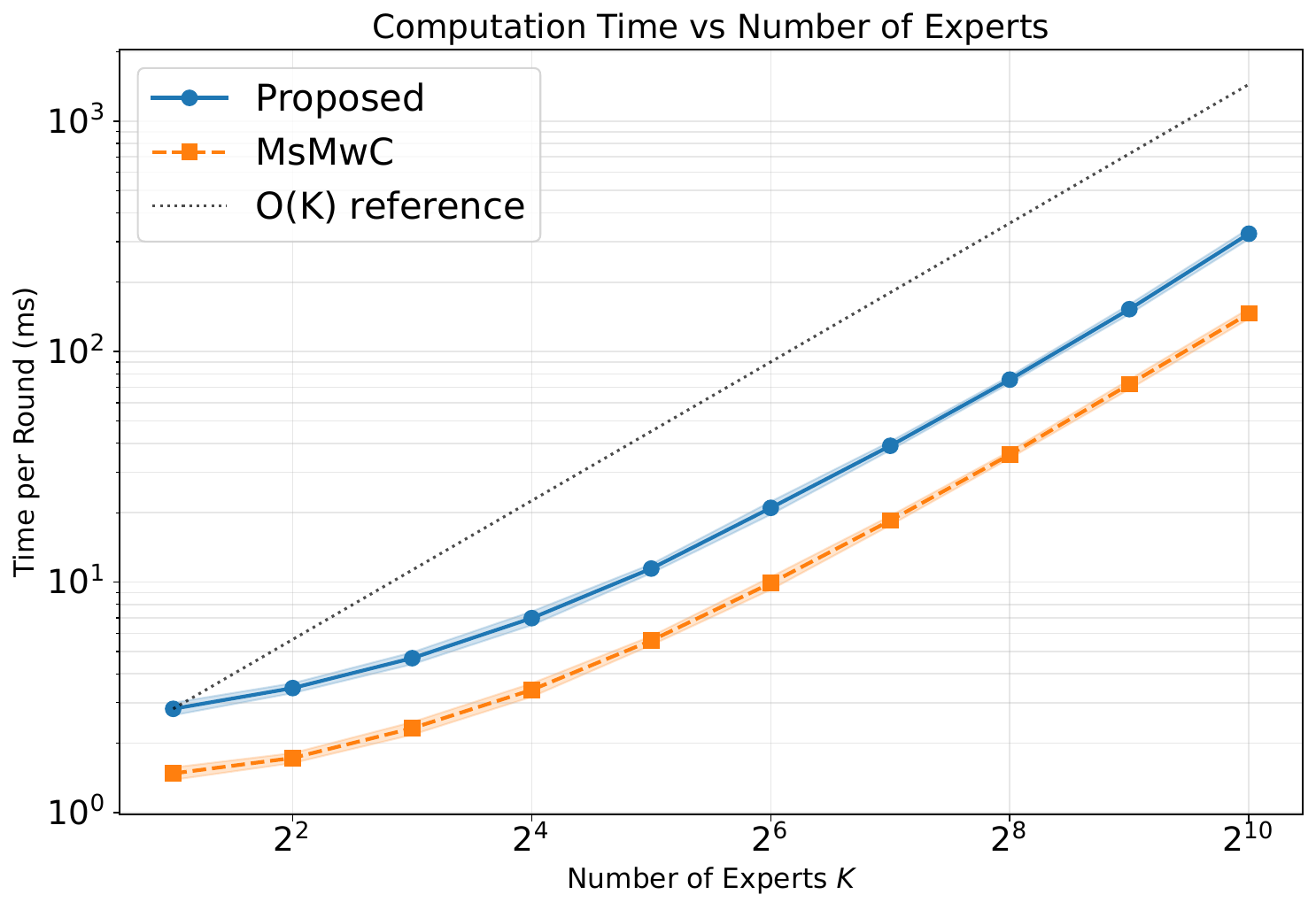}
    \Description{A log-log line plot of time per round in milliseconds versus the number of experts K, with K ranging from 2 to 1024. Two curves show the Proposed algorithm (blue circles) and MsMwC (orange squares), with thin shaded standard deviation bands. A grey dotted line indicates the O(K) linear reference. Both curves are roughly parallel to the O(K) reference at large K, with the Proposed algorithm consistently about two to three times slower than MsMwC in absolute terms. At K=2 the times are around 3 ms and 1.5 ms respectively; at K=1024 they reach about 300 ms and 150 ms.}
    \caption{
    Scalability of computation time with respect to the number $K$ of experts (log-log scale). The grey line represents the theoretical linear scaling reference $O(K)$. The shaded regions represent the standard deviation. The proposed algorithm scales linearly with $K$, exhibiting the same growth trend as MsMwC consistent with the $O(K)$ factor in $O(K \log^2 T)$.
    }
    \label{fig:time_scaling}
\end{figure}

\section{Numerical Experiments}\label{sec:experiments}

This section empirically evaluates the proposed algorithm to demonstrate its superior agility and robustness compared to existing tuning-free algorithms.
We first analyze the behavior of the safeguarded large learning rates under controlled synthetic distribution shifts to clarify the mechanism behind the resolved adaptation lag.
Subsequently, we validate the versatility of our approach across diverse real-world datasets.
Finally, we show that the computational overhead of the safeguard mechanism is negligible, ensuring scalability for real-world applications.

\subsection{Experimental Setup}

We evaluate the proposed algorithm against four tuning-free baselines.
The baselines comprise AdaNormalHedge.TV (ATV)~\citep{luo2015achieving}, Squint~\citep{koolen15second}, Coin betting for changing environments with AdaNormal potential (CBCE)~\citep{jun17a}, and MsMwC~\citep{chen21impossible}.\footnote{
To facilitate a clearer ablation study, we employ a single-layer version of MsMwC.
Since it can be transformed into a single-layer form without loss of theoretical guarantees, this setup allows us to isolate the effects of our aggressive learning rates.
}
These baselines are specifically selected because they share key characteristics with our approach: they handle an unknown total number of experts, require no hyperparameter tuning, and possess theoretical guarantees for dynamic regret.
All algorithms were implemented in Python and executed on a single CPU core to ensure a fair comparison of runtime.

To ensure a rigorous comparison, we utilize both controlled synthetic data and diverse real-world benchmarks.
For the synthetic evaluation, we employ the Rotated MNIST dataset~\citep{david2017mnist} to generate image streams containing three distinct types of distribution shifts: abrupt, incremental, and corruption.
This setup allows us to rigorously analyze the adaptation mechanism under controlled conditions.
To further validate the versatility of the proposed algorithm in practical applications, we utilize a suite of real-world datasets spanning various domains~\citep{ditzler2012incremental,yao2022wild,Hadi13,ikonomovska2011learning,harries1999splice2,blackard1999comparative,StreamMiningRepo,losing2016knn}.

Across all experiments, we maintain a consistent protocol to evaluate the online model selection performance.
We construct an environment where 100 distinct prediction models operate as experts for each dataset.
For the synthetic experiments, these experts consist of pre-trained models specializing in specific rotation angles.
In contrast, for the real-world datasets, we simulate a more practical scenario where the expert pool evolves over time.
Specifically, we sequentially introduce new models trained on recent data while phasing out degraded ones, thereby evaluating the algorithm's ability to manage a dynamic set of experts in non-stationary environments.
The detailed specifications of these datasets and hyperparameters are provided in \appref{app:experiments}.

We quantify predictive performance based on cumulative loss, which serves as a direct measure of the algorithm's adaptability.
For the synthetic experiments, we report the average cumulative loss and standard error to ensure statistical reliability.
For the real-world datasets, the cumulative losses are normalized such that the average cumulative loss of the MsMwC baseline equals 100 to facilitate a clear relative comparison across different domains.

\subsection{Analysis of Adaptation Mechanisms}

\subsubsection{Adaptation to Various Distribution Shifts}

The proposed algorithm consistently achieved the lowest average cumulative loss compared to all baselines across abrupt, incremental, and corruption drift scenarios.
As shown in \tabref{tab:synthetic_results}, our approach outperformed existing algorithms in every setting.
While MsMwC required hundreds of rounds to shift weights to the new optimal model due to its restricted learning rate ($O(1)$),
the proposed algorithm completed the adaptation within only a few rounds by utilizing large learning rates of scale $\Theta(T)$ (\figref{fig:ablation_study}, top row).
These results empirically demonstrate that our safeguard mechanism successfully resolves the trade-off between robustness and agility.

\subsubsection{Ablation Study and Parameter Sensitivity}

Our ablation study reveals that integrating the post-hoc penalty mechanism with optimistic prediction is indispensable for resolving adaptation lag.
We compared the behavior of the proposed algorithm with a variant lacking optimism (i.e., $m_t(i) = 0$).
As illustrated in the bottom row of \figref{fig:ablation_study}, in the configuration without optimism, the prediction error $\|\ell_t - m_t\|_\infty$ increased, causing the cumulative penalty to exceed the safety threshold early.
Consequently, the safeguard mechanism dynamically excluded large learning rates, forcing the algorithm into a slow adaptation regime similar to MsMwC.
In contrast, the full proposed algorithm maintained low cumulative penalties due to accurate optimistic predictions, thereby keeping large learning rates active for extended periods and enabling rapid weight redistribution.

Regarding the hyperparameter $M$, which defines the maximum size of the learning rate grid, our algorithm exhibits significant robustness to parameter selection.
We investigated the sensitivity of the cumulative loss by varying $M$ under the three drift settings (\figref{fig:sensitivity}).
When $M$ was set smaller than the theoretical requirement ($\lceil \log_2 T^2 \rceil$), the algorithm failed to include the optimal learning rate in its search range, resulting in increased loss.
However, we observed no performance degradation even when $M$ was set significantly larger than the required value.
This observation aligns with our theoretical analysis, which indicates that the dependence of $M$ on the regret bound is linear (i.e., $O(\log T)$).
Therefore, practitioners can safely set a sufficiently large $M$ to ensure stable performance without precise prior knowledge of the horizon $T$.

\subsection{Performance on Real-World Benchmarks}

The proposed algorithm demonstrated superior adaptability to real-world distribution drifts, achieving the lowest cumulative loss across all datasets tested.
\tabref{tab:real_world_results} presents the normalized average cumulative loss for diverse domains, covering both classification and regression tasks.
Our algorithm exhibited lower average cumulative loss than the baselines in every dataset, surpassing theoretically guaranteed algorithms.
Notably, this advantage holds true even in domains with complex noise patterns, such as electricity demand forecasting and flight delay prediction.

This consistent performance advantage stems from the algorithm's ability to balance agility and robustness in unpredictable environments.
Real-world data streams often contain a mixture of abrupt changes and noisy fluctuations.
By leveraging the safeguarded large learning rates, our algorithm can aggressively adapt to these changes without being destabilized by noise.
Consequently, the experimental results indicate that the proposed algorithm effectively resolves the trade-off between theoretical robustness and practical agility across a wide range of real-world domains.

\subsection{Computational Efficiency and Scalability}

Finally, we confirm that the proposed algorithm maintains high computational efficiency suitable for real-world applications, despite the sophisticated safeguard mechanism.
As shown in \tabref{tab:computation_time}, the average execution time on the synthetic datasets is $29.77 \pm 0.02$ ms per round.
Although this runtime is approximately twice that of MsMwC, the increase is directly attributable to the expanded search space of learning rates and the penalty mechanism.
Crucially, the absolute execution time remains under 30 ms, which is sufficiently fast to handle high-frequency data streams without causing latency bottlenecks.
Therefore, the proposed algorithm successfully balances the computational overhead required for enhanced agility with the low-latency execution essential for online processing.

Furthermore, the proposed algorithm demonstrates linear scalability with respect to the number of experts, validating our theoretical complexity analysis.
\figref{fig:time_scaling} illustrates the relationship between the number of experts $K$ and the average runtime per round, with the number of learning rate candidates $M$ fixed.
The results indicate that the runtime scales linearly with $K$, consistent with the theoretical time complexity of $O(K \log^2 T)$.
Notably, the ratio of the execution time between the proposed algorithm and MsMwC remains constant across all scales.
This observation confirms that the overhead of the safeguard mechanism is a constant factor and does not compromise the algorithm's scalability even in large-scale settings with extensive expert pools.


\section{Conclusion}

In this paper, we established a novel design paradigm that enables the utilization of large learning rates through post-hoc management.
The proposed algorithm, OOMD with safeguarded large learning rates, dynamically monitors unstable updates by introducing a cumulative penalty term $L$.
This mechanism overcomes the long-standing theoretical limitation that restricts learning rates to $O(1)$, allowing learning rates up to a scale of $\Theta(T)$ while maintaining state-of-the-art dynamic regret bounds.
Empirical evaluations demonstrate that the proposed algorithm resolves the adaptation lag by drastically reducing the number of rounds required to adapt to abrupt concept drifts.
Furthermore, our algorithm achieves lower cumulative loss compared to existing baselines even under incremental changes and corrupted distributions.

\bibliographystyle{ACM-Reference-Format}
\bibliography{references}

\appendix


\section{Generalization for Real-World Environments}\label{app:problem_setting}

In \secref{sec:preliminaries}, we defined the Tracking the Best Expert problem under the assumptions of a fixed number of experts and bounded losses ($\ell_t \in [0, 1]$) for theoretical clarity.
However, real-world stream data analysis necessitates adaptability to dynamic expert pools and unknown loss scales.
This appendix details the generalized problem setting and the corresponding extension of our algorithm used in the numerical experiments.

\subsection{Realistic Problem Setting}
Real-world environments often require the learner to incorporate new predictive models trained on recent data or to discard obsolete models.
To address this, we extend the formulation to handle a time-varying set of experts.
Let $K_t$ denote the set of indices of available experts at round $t$.
The size of this set, $|K_t|$, varies over time.
Consequently, the learner maintains a probability distribution $p_t$ over the dynamic simplex $\Delta_{|K_t|}$.

Furthermore, the losses of predictive models may exceed the unit interval.
Thus, we remove the assumption $\ell_t, m_t \in [0, 1]^K$.

\subsection{Loss Clipping for Unknown Loss Range}

\begin{algorithm}[tb]
    \caption{Loss Clipping Technique}
    \label{alg:loss_clipping}
    \begin{algorithmic}[1]
        \REQUIRE Algorithm $\mathcal{A}$, initial scale $B_0$, scaling rate $R$.
        \STATE Set $\tilde{B} = B_0$.
        \STATE Initialize $\mathcal{A}$ for loss range $[0, \tilde{B}]$.
        \FOR{$t = 1, 2, \dots, T$}
            \STATE Receive prediction vector $m_t$ and feed it to $\mathcal{A}$.
            \STATE Obtain decision $p_t$ from $\mathcal{A}$.
            \STATE Play $p_t$ and observe loss vector $\ell_t$.
            \STATE Let $B_t = \max_{s \in [t]} \|\ell_s - m_s\|_\infty$ and $\bar{\ell}_t = m_t + \frac{B_{t-1}}{B_t}(\ell_t - m_t)$.
            \STATE Feed $\bar{\ell}_t$ to $\mathcal{A}$ for weight update.
            \IF{$B_t > \tilde{B}$}
                \STATE Set $\tilde{B} = \max(R\tilde{B}, B_t)$.
                \STATE Initialize $\mathcal{A}$ for loss range $[0, \tilde{B}]$.
            \ENDIF
        \ENDFOR
  \end{algorithmic}
\end{algorithm}

\algoref{alg:loss_clipping} outlines the Loss Clipping technique designed to handle unknown loss ranges.
This wrapper mechanism ensures that the underlying algorithm $\mathcal{A}$ always receives loss vectors whose prediction errors $\|\ell_t - m_t\|_\infty$ are strictly bounded by the current scale $\tilde{B}$.
For algorithms that do not utilize optimistic predictions, we simply set the prediction vector $m_t$ to the zero vector.

This technique extends the theoretical guarantees of algorithms designed for bounded losses to arbitrary loss ranges.
When the underlying algorithm $\mathcal{A}$ requires a fixed range (e.g., $[0, 1]$), we normalize the input prediction and loss vectors by the current scale $\tilde{B}$ before feeding them to $\mathcal{A}$.
Although the dynamic rescaling involves algorithm restarts and the use of the surrogate loss $\bar{\ell}_t$, the theoretical analysis guarantees that the resulting overhead in the cumulative regret is negligible~\citep{chen21impossible,Cutkosky19}.

In our numerical experiments, we configured the hyperparameters based on the theoretical capabilities of each algorithm.
For the baselines other than MsMwC and the proposed algorithm, we set the initial scale $B_0 = 1$ and the scaling rate $R = 2$,
as these baselines assume a bounded loss range of $[0, 1]$.
In contrast, MsMwC and the proposed algorithm theoretically accommodate losses within the range $[0, T]$ without modification.
Therefore, we set $B_0 = T$ and $R = T$ for these two algorithms, consistent with the algorithm design specified in Theorem 8 of \citet{chen21impossible}.

\subsection{Generalized Proposed Algorithm}\label{app:general_alg}

\begin{algorithm}[tb]
    \caption{OOMD with Safeguarded Large Learning Rates for Unknown Total Number of Experts}
    \label{alg:proposed_general}
    \begin{algorithmic}[1]
        \REQUIRE Horizon $T \in \N{}$, Initial threshold $U_1$.
        \STATE Set $M = \lceil \log_2 T^2 \rceil$.
        \STATE Define learning rates $\eta(j) = \frac{2^j}{16T}$ for all $j \in [M]$.
        \STATE Initialize weights $w'_1(i, j) = \eta(j)^2$.
        \STATE Initialize penalties $L_1(j) = 0$ for all $j \in [M]$.
        \FOR{$t = 1, 2, \dots, T$}
            \STATE Receive set of experts $K_t$ and prediction vector $m_t \in \R{|K_t|}$.
            \STATE Let $\tilde{m}_t(i,j) = m_t(i)$ for all $i \in K_t$ and $j \in [M]$.
            \STATE Define $w'_t(i,j) = \eta(j)^2$ for all $i \in K_t \setminus \bigcup_{s \in [t-1]} K_s$.
            \STATE Let $W_t = \sum_{i \in K_t} \sum_{j \in [M]} w'_t(i,j)$ and $p'_t(i,j) = w'_t(i,j) / W_t$.
            \STATE Construct active set $\mathcal{J}_t = \{ j \in [M] \mid L_t(j) \le U_t \}$.
            \STATE Let $\psi_t(w) = \sum_{i \in K_t,\,j \in [M]} \frac{1}{\eta(j)}w(i,j)\log w(i,j)$ and $\varepsilon_t = \frac{1}{64 W_t T^2}$.
            \STATE Compute $\tilde{p}_t \in \argmin_{p \in \Omega_t} \langle \tilde{m}_t, p \rangle + D_{\psi}(p, p'_t)$, where $\Omega_t = \{ p \in \Delta_{|K_t| \times M} \mid p(i,j) \ge \varepsilon_t\,(\forall j \in \mathcal{J}_t),\ p(i,j) = \varepsilon_t\,(\forall j \notin \mathcal{J}_t) \}$.
            \STATE Compute final prediction $p_t(i) \propto \sum_{j \in \mathcal{J}_t} \tilde{p}_t(i,j)$.
            \STATE Play $p_t$ and observe loss vector $\ell_t \in \R{|K_t|}$.
            \STATE Let $\tilde{\ell}_t(i,j) = \ell_t(i)$ for all $i \in K_t$ and $j \in [M]$.
            \STATE Compute $\tilde{p}'_{t+1} \in \argmin_{p \in \Omega_t} \langle \tilde{\ell}_t + a_t, p \rangle + D_{\psi}(p, p'_t)$, with correction $a_t(i,j) = 32\eta(j)(\ell_t(i) - m_t(i))^2$.
            \STATE Update weights $w'_{t+1}(i,j)$ according to \eqref{eq:scaling_weight_update}
            \STATE Update penalties $L_{t+1}(j)$ according to \eqref{eq:penalty_update_general}.
            \STATE Update threshold as $U_{t+1} = \max(U_t, \|\ell_t - m_t\|_\infty)$.
        \ENDFOR
  \end{algorithmic}
\end{algorithm}

\algoref{alg:proposed_general} presents the extension of the proposed algorithm to the generalized setting where the set of experts and the loss scale change over time.
This algorithm adapts the weight update rule and the safeguard mechanism to handle the dynamic nature of the problem.

The algorithm restricts the weight update to the active set of experts $K_t$.
In each round, the algorithm first calculates the total weight $W_t$ of the currently available experts.
It then constructs a temporary probability distribution $p'_t$ by normalizing these weights.
After executing the standard OOMD update within this subspace to obtain $\tilde{p}'_{t+1}$, the algorithm scales the result back by $W_t$ as follows:
\begin{align}\label{eq:scaling_weight_update}
    w'_{t+1}(i,j) &= \left\{
    \begin{array}{ll}
        W_t \tilde{p}'_{t+1}(i,j) & \mathrm{if}\ i \in K_t \\
        w'_t(i,j) & \mathrm{otherwise}
    \end{array}\right.
\end{align}
This procedure ensures that the relative weights of active experts are updated correctly while preserving the stored weights of temporarily inactive experts.
Moreover, the penalty update is also restricted to the active experts:
\begin{equation}\label{eq:penalty_update_general}
    L_{t+1}(j) = L_t(j) + \sum_{i \in K_t} \ind[32 \eta(j)|r_t(i)| > 1] \tilde{p}_t(i, j) r_t(i)
\end{equation}

To accommodate unknown loss scales, we modify the safeguard mechanism.
Unlike the fixed threshold used in the basic setting, this generalized algorithm utilizes a dynamic threshold $U_t$ that tracks the maximum prediction error observed so far.
The learner considers a learning rate candidate $j$ active only if its cumulative penalty $L_t(j)$ remains within this adaptive bound $U_t$.

This generalized architecture preserves the theoretical robustness of the original approach.
Specifically, the following lemma establishes that the core performance guarantees remain valid even with a time-varying set of experts.

\begin{lemma}\label{lem:variable_experts}
Let $K_{total} = \bigcup_{t \in [T]} K_t$ be the union of all expert sets, identified with $[K]$ without loss of generality.
We assume two conditions regarding the expert set dynamics:
(1) any expert $i \in K_t \setminus K_{t-1}$ is strictly new, meaning $i \notin \bigcup_{s \in [t-1]} K_s$,
and (2) the total weight satisfies monotonicity, $W_{t-1} \le W_t$, for all $t$.
Then, for any interval $\{t_1, t_2\} \subseteq [T]$, any comparator $u \in \Delta_{K_{total}}$ such that its support is contained in $\bigcap_{t=t_1}^{t_2} K_t$, and any learning rate index $j^* \in \bigcap_{t=t_1}^{t_2} \mathcal{J}_t$,
\algoref{alg:proposed_general} achieves the following regret bound:
\begin{align*}
    \sum_{t=t_1}^{t_2} \langle \ell_t, p_t - u \rangle
    &\le \frac{1}{\eta(j^*)} \ln (64 W_{t_2+1} T) + 32\eta(j^*) \sum_{t=t_1}^{t_2} \sum_{i \in K_t} u(i) r_t(i)^2 \\
    &\quad+ \sum_{j \in [M]} \frac{q_{p'_{t_1}}(j) - q_{\tilde{p}'_{t_2+1}}(j)}{\eta(j)} \\
    &\quad+ \sum_{t=t_1}^{t_2} \sum_{j \in [M]} \Delta L_t(j) + O\left(\frac{t_2 - t_1}{T}\right),
\end{align*}
where, $q_p(j) = \sum_{i \in K} p(i, j)$ denotes the marginal probability of distribution $p$ for the learning rate layer $j$.
\end{lemma}

This lemma implies that \algoref{alg:proposed_general} achieves dynamic regret bounds consistent with those established in Theorems \ref{thm:switching_regret}, \ref{thm:switching_regret_benign}, and \ref{thm:path_length}.
Crucially, these guarantees hold for dynamic expert pools, with the bounds properly scaled by the magnitude of the losses.

The assumptions (1) and (2) on expert dynamics are grounded in practical implementation strategies common in data stream mining.
Assumption (1) treats each added expert as a unique instance, which reflects the standard practice of frequent model retraining.
Even if a previously removed model is reintroduced, assigning it a new unique identifier ensures it starts with a fresh weight and learning rate.
Assumption (2) is naturally satisfied by our pruning mechanism, which is motivated by computational efficiency.
In our experiments, we limit the expert pool to 100 members to maintain a constant computational overhead.
When this limit is reached, we remove experts with the smallest current weights.
Since these removed experts account for only a negligible fraction of the total weight, while incoming experts contribute a fixed initial weight, the total weight $W_t$ remains non-decreasing.
Thus, this common engineering practice for resource management simultaneously ensures the validity of our theoretical framework.


\section{Experimental Setup}\label{app:experiments}

\subsection{Algorithm Configurations}
In this subsection, we provide the detailed configurations for the proposed algorithm and the baseline methods used in our experiments.

\subsubsection{Common Settings}
To accommodate regression tasks where the range of loss values is unknown a priori, we applied the loss clipping technique (\algoref{alg:loss_clipping}) to all algorithms.
This ensures that the loss values processed by the algorithms lie in [0, 1], satisfying the requirements for the theoretical guarantees.

\subsubsection{Proposed Algorithm}
For the proposed algorithm, we used the following settings.

\paragraph{Learning Rate Grid and Threshold for Active Learning Rates}
We set the number of candidate learning rates to $M = 40$, and set the threshold for active learning rates to $U_1 = 1$.
This grid size satisfies the theoretically required condition: $M \ge \log_2(T^2)$.

\paragraph{Optimistic Prediction}
Recall that $K_t$ is the set of experts available in round $t$.
We used the following optimism:
\begin{align*}
    m_t(i) = m'_t(i) + \langle \ell_t - m'_t, p_t \rangle
\end{align*}
for all $i \in K_t$, where
\begin{align*}
    m'_t(i) &= \left\{
    \begin{array}{ll}
        \frac{1}{2}m'_{t-1}(i) + \frac{1}{2}\ell_{t-1}(i) & \mathrm{if}\ i \in K_{t-1} \\
        \tilde{m}_t & \mathrm{if}\ i \in K_t \setminus \bigcup_{s \in [t-1]} K_s \\
        m'_{t-1}(i) & \mathrm{otherwise}
    \end{array}\right.
\end{align*}
and
\begin{align*}
    \tilde{m}_t &= \left\{
    \begin{array}{ll}
        \frac{1}{2}\tilde{m}_{t-1} + \frac{1}{2}\langle \ell_{t-1}, p_{t-1} \rangle & \mathrm{if}\ t > 1 \\
        0 & \mathrm{otherwise}
    \end{array}\right..
\end{align*}
Note that we do not need to know $\langle \ell_t-m'_t, p_t \rangle$ for the decision of OOMD since adding a constant loss vector does not affect our decision.

The update rule of surrogate optimism $m'_t$ is based on algorithms for tracking the best linear predictor~\citep{cesa2006prediction,Herbster2001tracking} and the sleeping experts framework~\citep{Adamaskiy12,Koolen12}.
More precisely,
(1) $\{m'_t(i)\}_t$ is a sequence of the online gradient descent to minimize the squared error for $\ell_t(i)$ in round $t$, and
(2) the initial value of $m'_t(i)$ is derived from the losses by the algorithm.

\paragraph{Weight Usage for Expert Removal}
In the model management phase in the experiments using real-world datasets, the decision to remove degraded experts is based on the weights $w'_t$.
These correspond to the distribution prior to the optimistic update step, rather than the final weights $w_t$ used for prediction.

\subsubsection{Baselines}
We compared our approach against four state-of-the-art tuning-free algorithms.

\paragraph{MsMwC~\citep{chen21impossible}}
To facilitate a clearer ablation study, we employ a single-layer version of MsMwC for unknown loss range (see footnote 5 in \citet{chen21impossible}).
To ensure a fair comparison with the proposed algorithm, we set $T = 2^{20}$.
Moreover, MsMwC employed the same optimism and expert removal process as the proposed algorithm.
While the original MsMwC cannot be directly applied to scenarios where the total number of experts is unknown,
it is compatible with the weight scaling mechanism introduced in \appref{app:general_alg}.
Therefore, we integrated this weight scaling into MsMwC for our experiments on real-world datasets to accommodate dynamic expert pools.

\paragraph{ATV~\citep{luo2015achieving}, Squint~\citep{koolen15second}, and CBCE~\citep{jun17a}}
These algorithms are theoretically tuning-free.
Therefore, we utilized their standard implementations without explicit hyperparameter tuning.
Note that while Squint is originally designed for stationary environments, it was converted for non-stationary environments using the data-streaming technique~\citep{Hazan2007AdaptiveAF} in our experiments, aligning its setup with those of ATV and CBCE.

\subsection{Setup for Synthetic Datasets}
\subsubsection{Datasets and Experts Construction}
To construct a diverse pool of experts capable of handling rotational concept drift, we generated a set of predictive models trained on rotated images.
We employed logistic regression as the base model for all experts.
We prepared 100 distinct experts by training each model on the MNIST training data rotated by a specific angle.
These angles were uniformly distributed from 0 to 360 degrees with an increment of 3.6 degrees, thereby ensuring comprehensive coverage of all possible orientations.

\subsubsection{Protocol}
The experimental protocol involved a sequential processing of the data stream.
We first shuffled the original dataset to eliminate any inherent ordering.
The first 10,000 samples were designated for the pre-training phase to initialize the 100 expert models described above.
Following this initialization, the online learning process was executed over a horizon of 400 rounds.
In each round, the algorithms received a mini-batch consisting of 10 images.
The learner predicted the class labels for these images, and the performance was evaluated based on the zero-one loss averaged over the predictions within the current batch.
We conducted 500 independent trials for each drift scenario described below.

\subsubsection{Drift Scenarios}
To rigorously assess the algorithms under different non-stationary conditions, we simulated three distinct drift scenarios:
abrupt drift, incremental drift, and data corruption.

In the abrupt drift scenario, the distribution of the incoming data changes suddenly at fixed intervals.
We configured the target rotation angle of the stream to switch to a new, randomly selected angle every 100 rounds.
This setup tests the ability of the learner to quickly adapt to sudden shifts in the optimal expert.

In the incremental drift scenario, the data distribution evolves continuously over time.
We implemented this by rotating the incoming images by an additional 1 degree at every round, requiring the learner to track a smoothly changing target.

Finally, the corruption scenario evaluates robustness against noise.
In this setting, the data stream followed a fixed rotation, but with a 10\% probability, the incoming images were rotated by a random angle rather than the target angle.
This simulates the presence of impulsive noise or outliers in the data stream.

\subsection{Setup for Real-World Datasets}

\subsubsection{Dataset Descriptions}
In this section, we provide a comprehensive description of the eleven real-world datasets employed to evaluate the versatility and robustness of the proposed algorithm in non-stationary environments.
These datasets encompass a diverse range of domains, including document classification, image recognition, and sensor data analysis, spanning both classification and regression tasks.
We summarize the statistical specifications of each dataset in \tabref{tab:datasets}.
The batch size is chosen such that the total data consumption throughout the 400 rounds of OMS does not exceed the available dataset size.

\begin{table*}[h]
    \centering
    \caption{Summary of Real-World Datasets. For the text-based datasets arXiv and HuffPost, the number of dimensions indicates the size of the feature vectors (768) extracted using DistilBERT.}
    \label{tab:datasets}
    \begin{tabular}{l c c c c c}
        \toprule
        \textbf{Dataset} & \textbf{Samples} & \textbf{Dimensions} & \textbf{Task Type} & \textbf{Classes} & \textbf{Batch Size} \\
        \midrule
        Airlines & 539,383 & 7 & Classification & 2 & 500 \\
        arXiv & 2,057,952 & 768 & Classification & 172 & 3,000 \\
        Bikesharing & 17,379 & 12 & Regression & N/A & 24 \\
        Electricity & 45,312 & 8 & Classification & 2 & 48 \\
        Forest & 581,012 & 54 & Classification & 7 & 500 \\
        HuffPost & 63,907 & 768 & Classification & 41 & 100 \\
        Powersupply & 29,928 & 2 & Classification & 24 & 24 \\
        Rialto & 82,250 & 27 & Classification & 10 & 100 \\
        Temperature & 18,158 & 8 & Regression & N/A & 20 \\
        Weather & 18,159 & 8 & Classification & 2 & 20 \\
        Yearbook & 37,921 & 1024 & Classification & 2 & 50 \\
        \bottomrule
    \end{tabular}
\end{table*}

\paragraph{Airlines~\citep{ikonomovska2011learning}}
This dataset functions as a binary classification benchmark for identifying flight delays.
The task involves predicting whether a given flight will be delayed based on route information, schedule details, and carrier data.
This stream exhibits concept drift due to seasonal travel patterns, changing airline operations, and fluctuating airport traffic conditions over time.

\paragraph{arXiv~\citep{yao2022wild}}
This dataset represents a document classification task where the objective is to categorize research papers into their primary subject areas, such as physics, computer science, or mathematics.
We utilize the DistilBERT model to extract 768-dimensional feature vectors from the raw text, following protocols established in \citet{yao2022wild}.
The data stream naturally contains semantic shifts as research trends and terminology evolve within the scientific community.

\paragraph{Bikesharing~\citep{Hadi13}}
This dataset is derived from the Capital Bikesharing system in Washington, D.C., encompassing a two-year usage log.
This regression task requires the prediction of the hourly count of rented bicycles based on environmental and seasonal settings, including temperature, humidity, wind speed, and holiday status.
The data exhibits strong periodicity and concept drift driven by changing user behaviors and weather patterns.

\paragraph{Electricity~\citep{harries1999splice2}}
This dataset captures the fluctuation of electricity prices in the New South Wales (Australia) market.
The classification goal is to predict whether the NSW price will be higher or lower than the VIC (Victoria) price over a 24-hour period.
This dataset is characterized by non-stationary behavior influenced by demand dynamics, market supply changes, and time-dependent consumption patterns.

\paragraph{Forest~\citep{blackard1999comparative}}
The Forest Covertype dataset involves the classification of forest cover types, such as Spruce-Fir or Lodgepole Pine, based on cartographic variables including elevation, slope, and soil type.
While often treated as a static benchmark, we process the data sequentially to simulate a streaming environment where spatial dependencies translate into temporal distribution shifts as the sampling location moves across different geological regions.

\paragraph{HuffPost~\citep{yao2022wild}}
This dataset comprises news headlines from the years 2012 to 2018.
Similar to the arXiv setup, this is a multiclass classification problem where we predict the news category (e.g., politics, sports, entertainment).
We employ DistilBERT to generate high-dimensional embeddings for the text.
The stream reflects cultural and political shifts, introducing significant concept drift as the distribution of news topics changes over the years.

\paragraph{Powersupply~\citep{StreamMiningRepo}}
This dataset contains hourly power supply records from an Italian electricity company recorded over three years.
The learning objective is to classify the current power supply record into one of twenty-four hours of the day.
Concept drift in this stream arises from seasonal effects, weather variations, and differences between working days and weekends, making the mapping between power levels and time of day non-stationary.

\paragraph{Rialto~\citep{losing2016knn}}
This dataset consists of image features extracted from surveillance cameras monitoring classic buildings and bridges.
The classification task involves identifying specific structures or scene attributes under varying lighting and weather conditions.
The visual concepts drift due to natural illumination changes, seasonal weather effects, and camera degradations over time.

\paragraph{Weather~\citep{ditzler2012incremental}}
This dataset is sourced from the National Oceanic and Atmospheric Administration (NOAA), specifically containing daily measurements from the Offutt Air Force Base over a fifty-year period.
This binary classification task aims to predict the presence of rain precipitation based on features such as temperature, pressure, and wind speed.
The dataset allows for the evaluation of long-term climate changes and seasonal drifts.

\paragraph{Temperature}
This dataset utilizes the same feature set as the Weather dataset but is formulated as a regression problem.
The objective is to predict the maximum temperature for the following day.
This task evaluates the capability of the algorithm to track continuous numerical changes in meteorological patterns, providing a rigorous test for regression performance in the presence of seasonal drift.

\paragraph{Yearbook~\citep{yao2022wild}}
This dataset is an image classification benchmark composed of high school yearbook portraits spanning several decades.
The task is to predict the gender of the individual in the photo.
The dataset exhibits significant covariate shift caused by evolving fashion trends, hairstyles, and photographic styles throughout the 20th century.

\subsubsection{Protocol}
We evaluated the performance of the algorithms over a fixed horizon of 400 rounds for all real-world datasets.
The batch size for each round varies depending on the total size of the dataset:
specific values for each dataset are listed in \tabref{tab:datasets}.
The starting position of the stream was sampled uniformly at random from the available data points, under the constraint that sufficient data remained to complete the full 400 rounds.
To ensure the robustness of the results against the temporal location of the data, we conducted 100 independent trials with different starting positions for each dataset and reported the mean performance.
The loss function used for evaluation depends on the task type.
For classification tasks, we employed the average zero-one loss calculated over the samples in the current batch.
For regression tasks, we utilized the Mean Absolute Error (MAE).

\subsubsection{Dynamic Expert Pool}
The proposed algorithm and baselines operate with a dynamic set of experts.
We defined a sequence of ten rounds as a single time unit, referred to as a chunk.
The mechanism for adding new experts is designed to maintain a diverse set of models trained on different historical lengths.
Specifically, a new expert is trained and added to the pool whenever the number of accumulated chunks reaches a value of $k \times 2^n$, where $k$ is an odd integer.
At this timing, the new expert is trained using the most recent $2^{n+1}$ chunks.
In the special case where $k=1$, the expert is trained using the most recent $2^n$ chunks.
This exponential schedule ensures that the pool contains experts specializing in both short-term recent trends and long-term historical patterns.

To maintain computational efficiency, we imposed a maximum limit on the number of active experts.
Whenever the total number of experts exceeds 100, a pruning process is triggered.
This process identifies and removes the expert with the smallest current weight, thereby maintaining the pool size at or below 100.

\subsubsection{Experts Construction}
The configuration of the base learners constituting the expert pool varies according to the data modality.

For document datasets, specifically arXiv and HuffPost, we utilized DistilBERT for feature extraction as established in \citet{yao2022wild}.
The expert pool for these datasets consisted of Logistic Regression, Support Vector Machines (SVM), and Ridge Regression-based classifiers trained on the extracted features.
For the Ridge Regression-based classifiers, we determined the class prediction based on the sign of the continuous output value.

For the Yearbook image dataset, we employed the final layer vectors of a Convolutional Neural Network (CNN) as feature representations.
The architecture of this CNN was adopted from \citet{yao2022wild}.
Unlike standard pre-trained models, we trained this CNN from scratch during the model construction phase.
The expert pool included the direct output of this original CNN, as well as Logistic Regression, SVM, and Ridge Regression-based classifiers trained on the CNN features.

For other datasets, the expert configurations differed between classification and regression tasks.
For classification tasks, we employed Decision Trees with maximum depths of 2, 4, 6, and 8, in addition to Logistic Regression.
For regression tasks, the expert pool consisted of Decision Trees with the same depth settings (2, 4, 6, and 8) and Ridge Regression.


\section{Additional Experimental Results}

\begin{figure*}[tb]
    \centering
    \includegraphics[width=0.9\textwidth]{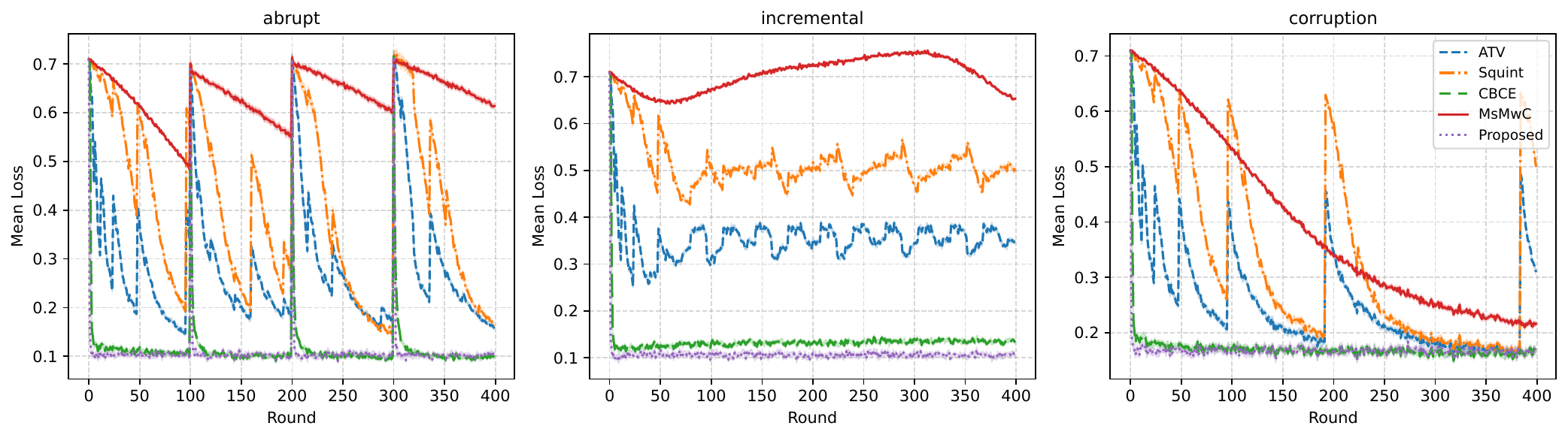}
    \Description{Three side-by-side line plots of mean loss versus round (0 to 400) on the Rotated MNIST dataset, comparing five algorithms: ATV, Squint, CBCE, MsMwC, and Proposed. Shaded bands show standard error. Left panel (abrupt drift): the Proposed and CBCE curves stay near 0.1 with brief spikes at the drift points around rounds 100, 200, and 300, while ATV, Squint, and MsMwC show large recovery transients each time. Middle panel (incremental drift): Proposed and CBCE stay flat near 0.1, ATV and Squint oscillate in the 0.3 to 0.5 band, and MsMwC drifts upward from 0.7 to 0.75. Right panel (corruption): Proposed and CBCE remain low near 0.15 to 0.2, ATV shows sharp upward spikes at intervals, Squint and MsMwC decay slowly from 0.7. The Proposed algorithm tracks CBCE closely across all three scenarios.}
    \caption{Trajectories of the mean loss on the Rotated MNIST dataset under abrupt, incremental, and corruption drift scenarios, comparing the proposed algorithm with four tuning-free baselines (ATV, Squint, CBCE, and MsMwC). The shaded regions represent the standard error. The proposed algorithm achieves the fastest recovery after each distribution shift, consistently attaining the lowest loss across all three drift types.}
    \label{fig:synthetic_comparison}
\end{figure*}

\begin{figure*}[tb]
    \centering
    \includegraphics[width=0.9\textwidth]{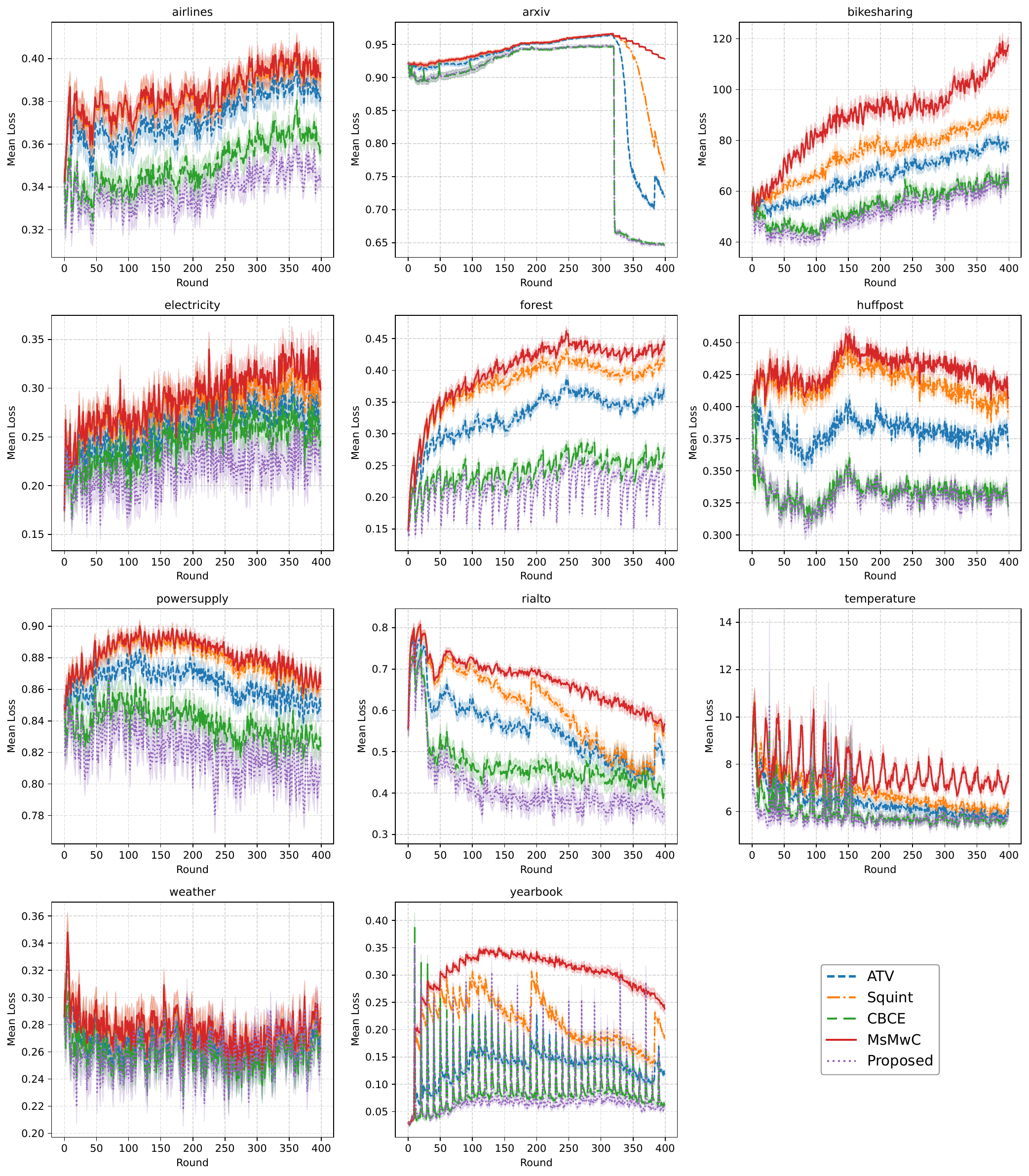}
    \Description{A four-row by three-column grid of line plots showing mean loss versus round (0 to 400) for eleven real-world datasets: airlines, arxiv, bikesharing, electricity, forest, huffpost, powersupply, rialto, temperature, weather, and yearbook. Each subplot compares five algorithms (ATV, Squint, CBCE, MsMwC, Proposed) with shaded standard error bands. Across nearly every dataset the Proposed algorithm (purple dotted) attains the lowest or near-lowest mean loss, typically tied with or slightly below CBCE (green dashed), while MsMwC (red) is the worst on most datasets and ATV and Squint sit in between. Notable patterns: on arxiv all methods rise to about 0.95 until around round 320, then ATV, Squint, CBCE, and Proposed drop sharply to 0.65 to 0.75 while MsMwC stays high; on bikesharing all curves drift upward over time with Proposed and CBCE growing most slowly; on yearbook the curves show periodic vertical spikes with Proposed attaining the lowest baseline near 0.05.}
    \caption{Trajectories of the mean loss on the eleven real-world datasets, comparing the proposed algorithm with four tuning-free baselines (ATV, Squint, CBCE, and MsMwC). The shaded regions represent the standard error. The proposed algorithm outperforms or matches all baselines across diverse classification and regression tasks, demonstrating its robustness and agility in practical non-stationary environments.}
    \label{fig:all_dataset_comparison}
\end{figure*}

This appendix presents the detailed temporal evolution of the cumulative loss for the experiments discussed in \secref{sec:experiments}.
While the main text reports the final cumulative loss in Tables \ref{tab:synthetic_results} and \ref{tab:real_world_results}, these supplementary figures illustrate the trajectory of the learner's performance throughout the entire horizon.
These visualizations clarify how the proposed algorithm adapts to distribution shifts over time compared to the baselines.

\figref{fig:synthetic_comparison} displays the learning curves for the Rotated MNIST dataset under three distinct drift scenarios.
In the abrupt drift setting (left), the proposed algorithm exhibits a sharp decrease in loss immediately after the distribution shifts.
This trajectory contrasts with MsMwC, which suffers from a prolonged adaptation lag due to its conservative learning rate constraints.
Additionally, the results for incremental drift and corruption confirm that the proposed algorithm consistently outperforms the baselines.

\figref{fig:all_dataset_comparison} plots the cumulative loss over time for the eleven real-world datasets.
The proposed algorithm consistently maintains the lowest loss trajectory across diverse domains, ranging from classification to regression tasks.
These trajectories demonstrate that the proposed approach effectively balances agility and robustness in practical applications.


\section{Missing proofs}\label{app:proof}

\subsection{Useful Lemmas}

\begin{lemma}[Lemma 15 of \citet{chen21impossible}]
    \label{lem:OMD_bound}
    Let $w_t \in \argmin_{w \in \mathcal{K}} \langle m_t, w \rangle + D_{\psi}(w, w'_t)$ and $w'_{t+1} \in \argmin_{w \in \mathcal{K}} \langle \ell_t, w \rangle + D_{\psi}(w, w')$ for some compact convex set $\mathcal{K} \subset \R{d}$, convex function $\psi_t$, arbitrary points $\ell_t, m_t \in \R{d}$, and a point $w' \in \R{K}$ such that $D_{\psi}(x, w')$ can be defined for all $x \in \mathcal{K}$.
    Then, for any $u \in \mathcal{K}$, we have
    \begin{align*}
        \langle \ell_t, w_t - u \rangle
        &\le \langle w_t - w'_{t+1}, \ell_t - m_t \rangle
        + D_{\psi}(u, w'_t) - D_{\psi}(u, w'_{t+1}) \\
        &\quad - D_{\psi}(w'_{t+1}, w_t) - D_{\psi}(w_t, w'_t).
    \end{align*}
\end{lemma}
\begin{remark}
    \lemref{lem:OMD_bound} relaxes the condition of $w'$ from the original version.
    Specifically, the original version assumes $w' \in \mathcal{K}$.
    We can prove the relaxed version without any modification of the existing proof.
    This relaxation is essential not only our analysis but also the existing proof of Theorem 8 of \citet{chen21impossible}.
\end{remark}

\begin{lemma}[Lemma 7 of \citet{luo2015achieving}]\label{lem:decomposition}
For any sequence of non-negative numbers $v_1, \dots, v_T$ with $v_0=v_{T+1}=0$, there exists a collection of intervals $\mathcal{U} = \{U_k\}_{k=1}^N$ and positive weights $\{a_k\}_{k=1}^N$ such that
\begin{equation}
    v_t = \sum_{k=1}^N a_k \ind[t \in U_k], \quad \forall t \in [T],
\end{equation}
and $\sum_{k=1}^N a_k = \sum_{t=1}^{T+1} [v_t - v_{t-1}]_+$,
where $[x]_+ = \max(x, 0)$.
\end{lemma}

\subsection{Proof of \lemref{lem:general_bound}}

\begin{proof}[Proof of \lemref{lem:general_bound}]
    We fix $j^* \in \bigcap_{t = t_1}^{t_2} \mathcal{J}_t$ arbitrarily.
    We define $\tilde{u}$ and $\bar{u}$ as follows:
    \begin{align*}
        \tilde{u}(i,j) &= \left\{
            \begin{array}{ll}
                u(i) & \mathrm{if}\ j = j^* \\
                0 & \mathrm{otherwise}
            \end{array}
        \right.\quad\mathrm{and} \\
        \bar{u}(i,j) &= \left\{
            \begin{array}{ll}
                (1 - \frac{1}{T^3})u(i) + \varepsilon & \mathrm{if}\ j = j^* \\
                \varepsilon & \mathrm{otherwise}
            \end{array}
        \right.
    \end{align*}
    Note that $\bar{u} \in \Omega_t$ for all $t$.

    By a basic calculation, we have $\langle \hat{\ell}_t, w_t \rangle = \langle \ell_t, p_t \rangle$ for all $t$.
    Thus, we have
    \begin{align*}
        \sum_{t = t_1}^{t_2} \langle \ell_t, p_t - u \rangle
        &= \langle \tilde{\ell}_t, w_t - \tilde{u} \rangle \\
        &= \sum_{t = t_1}^{t_2} \langle \tilde{\ell}_t, w_t - \bar{u} \rangle 
        + \sum_{t = t_1}^{t_2} \langle \tilde{\ell}_t, \tilde{u} - \bar{u} \rangle \\
        &\le \sum_{t = t_1}^{t_2} \langle \tilde{\ell}_t, w_t - \bar{u} \rangle 
        + \frac{2(t_2 - t_1)}{T^3} \|\ell_t\|_\infty.
    \end{align*}
    We will consider the first term in the right-hand side.
    
    Using \lemref{lem:OMD_bound}, we have
    \begin{align*}
        \sum_{t = t_1}^{t_2}  \langle \tilde{\ell}_t, w_t - \bar{u} \rangle
        &\le \sum_{t = t_1}^{t_2}  (D_{\psi}(\bar{u}, w'_t) - D_{\psi}(\bar{u}, w'_{t+1})) \\
        &\quad + \sum_{t = t_1}^{t_2} A_t,
    \end{align*}
    where $A_t = \langle \tilde{\ell}_t - \tilde{m}_t + a_t, w_t - w'_{t+1} \rangle - D_{\psi}(w'_{t+1}, w_t) - \langle a_t, w_t - \bar{u} \rangle$.
    We will bound the terms on right-hand side separately.

    We consider the first term.
    Recall that $w'_t(i,j) \ge \varepsilon$ and $q_t(j) = \sum_{i \in [K]} w'_t(i,j)$.
    Then, we have
    \begin{align*}
        &\quad \sum_{t = t_1}^{t_2} (D_{\psi}(\bar{u}, w'_t) - D_{\psi}(\bar{u}, w'_{t+1})) \\
        &= D_{\psi}(\bar{u}, w'_{t_1}) - D_{\psi}(\bar{u}, w'_{t_2+1}) \\
        &= \sum_{i \in [K],\, j \in [M]} \frac{\bar{u}(i,j)}{\eta(j)} \log\frac{w'_{t_2+1}(i,j)}{w'_{t_1}(i,j)}
        + \sum_{j \in [M]} \frac{q_{t_1}(j) - q_{t_2+1}(j)}{\eta(j)} \\
        &\le \frac{4\log(KMT)}{\eta(j^*)}
        + \sum_{j \in [M]} \frac{4\log(KMT)}{\eta(j)MT^3}
        + \sum_{j \in [M]} \frac{q_{t_1}(j) - q_{t_2+1}(j)}{\eta(j)} \\
        &\le \frac{4\log(KMT)}{\eta(j^*)}
        + \frac{32\log(KMT)}{T^2}
        + \sum_{j \in [M]} \frac{q_{t_1}(j) - q_{t_2+1}(j)}{\eta(j)},
    \end{align*}
    where the second inequality is derived from $\sum_{j \in [M]} 1/\eta(j) \le 8MT$.

    Next, we consider the second term.
    We fix $t \in [T]$ arbitrarily.
    Let $v_t = \tilde{\ell}_t - \tilde{m}_t + a_t$ and $w^* \in \argmax_{w \in \R{K \times M}_+} \langle v_t, w_t - w \rangle - D_{\psi}(w, w_t)$,
    where $\R{d}_+ = \{ x \in \R{K \times M} \mid \forall i \in [d],\ x(i) \ge 0 \}$.
    Then, we have
    \begin{align*}
        &\quad A_t + \langle a_t, w_t - \bar{u} \rangle \\
        &= \langle v_t, w_t - w'_{t+1} \rangle - D_{\psi}(w'_{t+1}, w_t) \\
        &\le \langle v_t, w_t - w^* \rangle - D_{\psi}(w^*, w_t) \\
        &= \langle \nabla\psi(w_t) - \nabla\psi(w^*), w_t - w^* \rangle - D_{\psi}(w^*, w_t) \\
        &= \sum_{i \in [K],\, j \in [M]} \frac{w_t(i,j)}{\eta(j)} \left( \eta(j) v_t(i,j) - 1 + e^{-\eta(j) v_t(i,j)} \right).
    \end{align*}
    We will fix $(i,j)$ arbitrarily and discuss two cases based on the magnitude of the update.

    \paragraph{Case 1: small learning rate ($32\eta(j)|r_t(i)| \le 1$)}
    In this regime, the existing analysis holds.
    Specifically, we have
    \begin{align*}
        \eta(i)|v_t(i)| \le (1 + 32\eta(i)|r_t(i)|)\eta(i)|r_t(i)| \le 1.
    \end{align*}
    Therefore, since $e^{-x} - 1 + x \le x^2$ for $x \ge -1$, we obtain
    \begin{align*}
        &\quad \frac{w_t(i,j)}{\eta(j)} \left( \eta(j) v_t(i,j) - 1 + e^{-\eta(j) v_t(i,j)} \right) \\
        &\le \eta(j) w_t(i,j) v_t(i,j)^2
        \le 4 \eta(j) w_t(i,j) r_t(i)^2.
    \end{align*}

    \paragraph{Case 2: large learning rate ($32\eta(j)|r_t(i)| > 1$)}
    In this case, the condition implies
    \begin{align*}
        v_t(i,j) \ge (32\eta(j)|r_t(i)| - 1)|r_t(i)| > 0.
    \end{align*}
    Thus, we have
    \begin{align*}
        &\quad \frac{w_t(i,j)}{\eta(j)} \left( \eta(j) v_t(i,j) - 1 + e^{-\eta(j) v_t(i,j)} \right) 
        \le \frac{w_t(i,j)}{\eta(j)} \eta(j) v_t(i,j) \\
        &= 32 \eta(j) w_t(i,j) r_t(i)^2 + w_t(i,j) r_t(i).
    \end{align*}

    Summing over all $(i,j)$ and combining these two cases, we obtain
    \begin{align*}
        A_t
        &\le \sum_{i \in [K],\, j \in [M]} 4 \eta(j) w_t(i,j) r_t(i)^2
        - \langle a_t, w_t - \bar{u} \rangle \\
        &\quad + \sum_{i \in [K],\, j \in [M]} \ind[32\eta(j)|r_t(i)| > 1] w_t(i,j) r_t(i) \\
        &\le \langle a_t, \bar{u} \rangle
        + \sum_{i \in [K],\, j \in [M]} \ind[32\eta(j)|r_t(i)| > 1] w_t(i,j) r_t(i) \\
        &\le 32\eta(j^*) \sum_{i \in [K]} u(i)r_t(i)^2 + \frac{32\max_{t \in [T]} \|\ell_t - m_t\|_\infty^2}{MT^3} \sum_{j \in [M]} \eta(j) \\
        &\quad + \sum_{i \in [K],\, j \in [M]} \ind[32\eta(j)|r_t(i)| > 1] w_t(i,j) r_t(i) \\
        &\le 32\eta(j^*) \sum_{i \in [K]} u(i)r_t(i)^2 + \frac{256\max_{t \in [T]} \|\ell_t - m_t\|_\infty^2}{T^2}  \\
        &\quad + \sum_{i \in [K],\, j \in [M]} \ind[32\eta(j)|r_t(i)| > 1] w_t(i,j) r_t(i).
    \end{align*}
\end{proof}

\subsection{Switching Regret Bound}

\subsubsection{Proof of \thmref{thm:switching_regret}}
We first decomposes the switching regret into the sum over intervals $\mathcal{I}_s = [t_s, t_{s+1}-1]$.
Applying \lemref{lem:general_bound} to each interval $s$ with a specific choice of $j^*_s$, we have
\begin{align*}
    &\quad \sum_{s \in [S]} \sum_{t \in \mathcal{I}_s} \langle \ell_t, p_t - u_s \rangle \\
    &\le \sum_{s \in [S]} \left( \frac{4 \log(KMT)}{\eta(j^*_s)} + 32\eta(j^*_s) Q_s \right) \nonumber \\
    &\quad + \sum_{s \in [S]} \sum_{j \in [M]} \frac{q_{t_s}(j) - q_{t_{s+1}}(j)}{\eta(j)}
    + \sum_{j \in [M]} L_{T+1}(j)
    + \sum_{s=1}^S \xi(\mathcal{I}_s),
\end{align*}
where $\xi(\mathcal{I}) \in O(T^{-1})$ is a negligible error term.

For the second term in the right-hand side, we have
\begin{align*}
    \sum_{s \in [S]} \sum_{j \in [M]} \frac{q_{t_s}(j) - q_{t_{s+1}}(j)}{\eta(j)} &= \sum_{j \in [M]} \frac{1}{\eta(j)} (q_{t_1}(j) - q_{t_{S+1}}(j)) \nonumber \\
    &\le \sum_{j \in [M]} \frac{q_{1}(j)}{\eta(j)}.
\end{align*}
Using the initialization $w'_1(i,j) \propto \eta(j)^2$, we have $q_1(j) = \frac{K \eta(j)^2}{\sum_{j'} K \eta(j')^2}$.
Thus, we obtain
\begin{align*}
    \sum_{j \in [M]} \frac{q_{1}(j)}{\eta(j)}
    = \frac{\sum_{j \in [M]} \eta(j)}{\sum_{j \in [M]} \eta(j)^2}
    = O(T^{-1}).
\end{align*}

Then, we discuss the third term.
We fix $j \in [M]$ arbitrarily.
Let $t(j)$ be the last round such that $j \in \mathcal{J}_t$.
Then, we have
\begin{align*}
    L_{T+1}(j)
    &\le L_{t(j)}(j) + \sum_{i \in [K]} w_{t(j)}(i,j) \|\ell_t - m_t\|_\infty + \sum_{t = t(j)+1}^T \varepsilon \\
    &\le 1 + 1 + O(1/T^2)
    = O(1).
\end{align*}
Thus, we have $\sum_{j \in [M]} L_{T+1}(j) = O(\log T)$.

Then, we consider the optimal choice of learning rates.
We define the objective function for interval $s$ as $f_s(\eta) = \frac{4 \log(KMT)}{\eta} + 32 \eta Q_s$.
The optimal unconstrained learning rate is $\eta^*_s = \sqrt{\frac{\log(KMT)}{8 Q_s}}$.
We analyze the choice of $j^*_k$ from the discrete grid $\mathcal{H} = \{ \eta(j) \mid j \in [M] \}$.
Since $\|\ell_t - m_t\|_\infty \le 1$ and the grid has a geometric ratio of 2 (i.e., $\eta(j+1) / \eta(j) = 2$),
there exists an index $j^*_s$ such that
\begin{align*}
    \eta(j^*_s) \le \min\left( \frac{1}{32}, \eta^*_s \right) \le 2\eta(j^*_s).
\end{align*}
Note that $j^*_s \in \mathcal{J}_t$ for all $t \in [T]$.
Substituting this learning rate, we have
\begin{align*}
    \frac{4 \log(KMT)}{\eta(j^*_s)} + 32 \eta(j^*_s) Q_s
    &\le O\left( \sqrt{\log(KT) Q_s} + \log(KT) \right).
\end{align*}
\qed

\subsubsection{Proof of \thmref{thm:switching_regret_benign}}
The difference from the proof of \thmref{thm:switching_regret} is only in the choice of learning rates.
We fix $s \in [S]$ arbitrarily.
Choosing $\eta(j^*_s)$ for our regret bound, we obtain
\begin{align*}
    \frac{4 \log(KMT)}{\eta(j^*_s)} + 32 \eta(j^*_s) Q_s
    &\le O\left( \sqrt{\log(KT) Q_s} \right).
\end{align*}
\qed

\subsection{Path-Length Regret Bound}

The following theorem generalizes the path-length regret bounds to both worst-case and benign scenarios.
Note that \thmref{thm:path_length} can be recovered as a direct consequence of this result.

\begin{theorem}
Suppose that $\|\ell_t - m_t\|_\infty \le 1$ for all $t \in [T]$.
Let $u_{1:T} = (u_1, \dots, u_T)$ be any sequence of comparators where $u_t \in \Delta_K$.
Let $V(u_{1:T})$ denote the path-length of the comparator sequence.
Let $Q(u_{1:T})$ denote the cumulative squared prediction error of the comparator sequence.
Then, \algoref{alg:proposed} ensures the following regret bound:
\begin{equation*}
    R_T(u_{1:T}) \le O\left( \sqrt{V(u_{1:T}) Q(u_{1:T}) \log(K T)} + V(u_{1:T}) \log(K T) \right).
\end{equation*}
More specifically, the regret satisfies:
\begin{align*}
    R_T(u_{1:T}) \le \inf_{j^* \in \mathcal{J}^*} \left\{ \frac{V(u_{1:T}) \log(1/\varepsilon)}{\eta(j^*)} + 32 \eta(j^*) Q(u_{1:T}) \right\} + O(\log T),
\end{align*}
where $\mathcal{J}^* = \bigcap_{t=1}^T \mathcal{J}_t$ is the set of learning rates valid for all rounds.
\end{theorem}

\subsubsection{Proof of \thmref{thm:path_length}}

The proof relies on decomposing the comparator sequence into weighted intervals and analyzing the regret using the properties of OOMD.

\paragraph{Decomposition of the Comparator Sequence}

We first decompose the sequence $u_{1:T}$ into a linear combination of static interval comparators.
We utilize \lemref{lem:decomposition}.

Applying this lemma to the weight sequence $u_{1:T}(i)$ for each expert $i \in [K]$:
\begin{equation*}
    u_t(i) = \sum_{k=1}^{N_i} a_{k,i} \ind[t \in U_{k,i}],
\end{equation*}
where $U_{k,i} = \{s_{k,i}, \dots, e_{k,i}\}$. Summing the weights yields the path-length relationship:
\begin{align*}
    \sum_{i=1}^K \sum_{k=1}^{N_i} a_{k,i}
    &= \sum_{i=1}^K \sum_{t=1}^{T+1} [u_t(i) - u_{t-1}(i)]_+ \\
    &= \frac{1}{2} \sum_{t=1}^{T+1} \|u_t - u_{t-1}\|_1
    = \frac{1}{2} V(u_{1:T}).
\end{align*}

\paragraph{Regret Decomposition}
Fix an optimal learning rate index $j^* \in \bigcap_t \mathcal{J}_t$. We implicitly compare against the extended comparator sequence $\tilde{u}_t$ concentrated on layer $j^*$. Using the decomposition, we rewrite the total regret:
\begin{align}\label{eq:regret_decomp}
    R_T(u_{1:T}) &= \sum_{t=1}^T \langle \ell_t, p_t \rangle - \sum_{t=1}^T \sum_{i=1}^K u_t(i) \ell_t(i) \nonumber \\
    &= \sum_{i=1}^K \sum_{k=1}^{N_i} a_{k,i} \sum_{t \in U_{k,i}} (\langle \ell_t, p_t \rangle - \ell_t(i)) \nonumber \\
    &= \sum_{i=1}^K \sum_{k=1}^{N_i} a_{k,i} \sum_{t=s_{k,i}}^{e_{k,i}} \langle \tilde{\ell}_t, w_t - e_{i, j^*} \rangle.
\end{align}
This effectively decomposes the regret into a weighted sum of regrets against static experts $e_{i, j^*}$ on intervals $U_{k,i}$.

\paragraph{Potential Analysis and Telescoping Sums}
Recall that $r_t = \ell_t - m_t$ and $\Delta L_t(j) = \sum_{i \in [K]} \ind[32\eta(j)|r_t(i)| > 1] w_t(i,j) r_t(i)$.
For any interval $\{s, \dots, e\}$ and comparator $v$, we use \lemref{lem:general_bound}.
The lemma ensures that for any valid learning rate $j^* \in \bigcap_{t=s}^e \mathcal{J}_t$, the regret is bounded by:
\begin{align*}
    \sum_{t=s}^e \langle \tilde{\ell}_t, w_t - v \rangle
    &\le D_\psi(v, w'_s) - D_\psi(v, w'_{e+1}) + \sum_{t=s}^e 32 \eta(j^*) r_t(i)^2 \\
    &+ \sum_{t=s}^{e} \sum_{j \in [M]} \Delta L_t(j)
    + O(1/T^2).
\end{align*}

Substituting $v = e_{i, j^*}$ and using the property of the potential function $\psi(w) = \sum_{i,j} \frac{w(i,j)}{\eta(j)} \log w(i,j)$, the Bregman divergence decomposes into a logarithmic term and a linear term:
\begin{equation*}
    D_\psi(e_{i, j^*}, w) = \frac{1}{\eta(j^*)} \log \frac{1}{w(i, j^*)} + W(w) - \frac{1}{\eta(j^*)},
\end{equation*}
where $W(w) = \sum_{x,y} \frac{w(x,y)}{\eta(y)}$. Thus, the term $D_\psi(v, w'_s) - D_\psi(v, w'_{e+1})$ becomes:
\begin{equation*}
    \frac{1}{\eta(j^*)} \log \frac{w'_{e+1}(i, j^*)}{w'_s(i, j^*)} + (W(w'_s) - W(w'_{e+1})).
\end{equation*}
Summing this bound over all experts $i$ and all decomposition intervals $k=1,\dots,N_i$ with weights $a_{k,i}$, we analyze the two components separately.

First, consider the sum of linear potential terms, denoted by $\mathcal{S}_{\text{lin}}$:
\begin{equation*}
    \mathcal{S}_{\text{lin}} = \sum_{i=1}^K \sum_{k=1}^{N_i} a_{k,i} (W(w'_{s_{k,i}}) - W(w'_{e_{k,i}+1})).
\end{equation*}
We rearrange this sum by grouping terms according to the time index $t$. The term $W(w'_t)$ appears with a coefficient $+a_{k,i}$ if an interval starts at $t$ (i.e., $s_{k,i}=t$) and $-a_{k,i}$ if an interval ends at $t-1$ (i.e., $e_{k,i}=t-1$). Let $C_t$ denote the aggregate coefficient for $W(w'_t)$. Then:
\begin{equation*}
    C_t = \sum_{i=1}^K \left( \sum_{k: s_{k,i}=t} a_{k,i} - \sum_{k: e_{k,i}=t-1} a_{k,i} \right).
\end{equation*}
From the decomposition property $u_t(i) = \sum_k a_{k,i} \ind[t \in U_{k,i}]$, we have the relation $u_t(i) - u_{t-1}(i) = \sum_{k: s_{k,i}=t} a_{k,i} - \sum_{k: e_{k,i}=t-1} a_{k,i}$. Summing over all experts $i$, we get:
\begin{equation*}
    C_t = \sum_{i=1}^K (u_t(i) - u_{t-1}(i)) = \sum_{i=1}^K u_t(i) - \sum_{i=1}^K u_{t-1}(i).
\end{equation*}
Since $u_t \in \Delta_K$ is a probability distribution, $\sum_{i=1}^K u_t(i) = 1$ for all $1 \le t \le T$. With the convention $u_0 = u_{T+1} = \mathbf{0}_K$, we have $\sum_{i=1}^K u_0(i) = 0$ and $\sum_{i=1}^K u_{T+1}(i) = 0$. Therefore:
\begin{align*}
    C_1 &= 1 - 0 = 1, \\
    C_t &= 1 - 1 = 0 \quad \text{for } 2 \le t \le T, \\
    C_{T+1} &= 0 - 1 = -1.
\end{align*}
Consequently, the intermediate terms cancel out, yielding
\begin{equation*}
    \mathcal{S}_{\text{lin}} = 1 \cdot W(w'_1) - 1 \cdot W(w'_{T+1}) \le W(w'_1).
\end{equation*}
The initial potential $W(w'_1) = \sum_{i,j} \frac{w'_1(i,j)}{\eta(j)}$ is $O(1)$ by the initialization of $w'_1$ (as discussed in our proof of \thmref{thm:switching_regret}).

Next, we analyze the sum of logarithmic terms, denoted by $\mathcal{S}_{\text{log}}$:
\begin{equation*}
    \mathcal{S}_{\text{log}} = \frac{1}{\eta(j^*)} \sum_{i=1}^K \sum_{k=1}^{N_i} a_{k,i} \log \frac{w'_{e_{k,i}+1}(i, j^*)}{w'_{s_{k,i}}(i, j^*)}.
\end{equation*}
Similarly rearranging the sum by time $t$, the coefficient for $\log w'_t(i, j^*)$ is
\begin{equation*}
    \sum_{k: e_{k,i}=t-1} a_{k,i} - \sum_{k: s_{k,i}=t} a_{k,i} = -(u_t(i) - u_{t-1}(i)).
\end{equation*}
Thus, we have
\begin{equation*}
    \mathcal{S}_{\text{log}} = \frac{1}{\eta(j^*)} \sum_{t=1}^{T+1} \sum_{i=1}^K (u_{t-1}(i) - u_t(i)) \log w'_t(i, j^*).
\end{equation*}
Since the weights are lower-bounded by $\varepsilon$, $|\log w'_t(i, j^*)| \le \log(1/\varepsilon)$. We bound the sum using the $L_1$ norm:
\begin{equation}
    \mathcal{S}_{\text{log}} \le \frac{\log(1/\varepsilon)}{\eta(j^*)} \sum_{t=1}^{T+1} \sum_{i=1}^K |u_{t-1}(i) - u_t(i)| = \frac{V(u_{1:T}) \log(1/\varepsilon)}{\eta(j^*)}.
\end{equation}

Then, we discuss the penalty due to choosing large learning rates:
\begin{equation*}
    \mathcal{S}_{pen}
    = \sum_{i=1}^K \sum_{k=1}^{N_i} a_{k,i} \left( \sum_{t=s_{k,i}}^{e_{k,i}} \sum_{j \in [M]} \Delta L_t(j) \right).
\end{equation*}
We swap the order of summation.
We sum over $j$ and $t$ first:
\begin{equation*}
    \mathcal{S}_{pen}
    = \sum_{j \in [M]} \sum_{t=1}^T \Delta L_t(j) \left( \sum_{i=1}^K \sum_{k=1}^{N_i} a_{k,i} \ind[t \in U_{k,i}] \right).
\end{equation*}
Recall that $u_t(i) = \sum_{k=1}^{N_i} a_{k,i} \ind[t \in U_{k,i}]$. Therefore, the term in the parenthesis is simply $\sum_{i=1}^K u_t(i)$.
Since $u_t$ is a valid probability distribution in $\Delta_K$, $\sum_{i=1}^K u_t(i) = 1$.
Therefore, we have
\begin{equation*}
    \mathcal{S}_{pen}
    = \sum_{j \in [M]} \sum_{t \in [T]} \Delta L_t(j) \cdot 1
    = \sum_{j \in [M]} L_{T+1}(j).
\end{equation*}
From the proof of \thmref{thm:switching_regret}, we can show that for any $j$, $L_{T+1}(j) = O(1)$.
Summing over $M$ layers, we obtain
\begin{equation*}
    \mathcal{S}_{pen} = O(M) = O(\log T).
\end{equation*}

\paragraph{Final Bound}
Combining the terms, we have
\begin{equation*}
    R_T(u_{1:T}) \le \frac{V(u_{1:T}) \log(1/\varepsilon)}{\eta(j^*)} + 32 \eta(j^*) Q(u_{1:T}) + O(\log T).
\end{equation*}
We select the optimal index $j^*$ from the grid $\mathcal{H} = \{ \eta(j) = \frac{2^j}{16T} \}$.
Let $\eta^* = \sqrt{\frac{V \log(1/\varepsilon)}{32 Q}}$.
Since the grid is geometric, there exists a valid $j^*$ such that $\eta(j^*) \approx \min\{1/32, \eta^*\}$.
Substituting this yields the bound:
\begin{equation*}
    R_T(u_{1:T}) = O\left( \sqrt{V(u_{1:T}) Q(u_{1:T}) \log(KT)} + V(u_{1:T}) \log(KT) \right).
\end{equation*}
This concludes the proof. \qed

\subsection{Proof of \thmref{thm:safeguard_necessity}}

We construct a specific adversarial scenario to demonstrate that without the safeguard mechanism, the algorithm suffers linear regret.
We consider a setting with $K=2$ experts:
Expert 1 (suboptimal) and Expert 2 (optimal).

\paragraph{Setup}
We define the sequence of loss vectors $\ell_t$ and prediction vectors $m_t$ for all rounds $t \in [T]$.
The loss vector is fixed as $\ell_t = (1, 0.5)$.
Expert 1 consistently incurs a loss of $1$, while Expert 2 incurs a loss of $0.5$.
Thus, Expert 2 is the optimal expert.
The prediction vector is fixed as $m_t = (0, 0.5)$.
This prediction is misleading: it incorrectly predicts zero loss for the suboptimal Expert 1 while correctly predicting the loss for Expert 2.

\paragraph{KKT Conditions}
The update rule for $w_t$ minimizes the linearized objective subject to the simplex constraint and the lower bound constraint $w_t(i,j) \ge \varepsilon$.
By the KKT conditions, the updated weights take the form
\begin{align*}
    w_t(i,j) = w'_t(i,j) \exp\left(\eta(j)(-m_t(i) + \lambda + \mu(i,j))\right),
\end{align*}
where $\lambda$ is the Lagrange multiplier for the sum-to-one constraint $\sum_{i,j} w_t(i,j) = 1$,
and $\mu(i,j) \ge 0$ is the multiplier for the lower bound constraint.
The complementary condition implies $\mu(i,j) = 0$ whenever $w_t(i,j) > \varepsilon$.

\paragraph{Upper Bound on $\lambda$}
We first establish that $\lambda$ must be small.
Consider the update for Expert 1. Since $m_t(1) = 0$, the weight update becomes
\begin{align*}
    w_t(1,j) = w'_t(1,j) \exp(\eta(j)(\lambda + \mu(1,j))).
\end{align*}
Since $w_t(1,j) \le 1$ holds for any $j$, we specifically look at the index $\hat{j}$ with the largest learning rate.
We have:
\begin{align*}
    w'_t(1,\hat{j}) \exp(\eta(\hat{j})\lambda) \le w_t(1,\hat{j}) \le 1.
\end{align*}
Taking the logarithm yields $\lambda \le \frac{1}{\eta(\hat{j})} \log \frac{1}{w'_t(1,\hat{j})}$.
Since $\eta(\hat{j}) = \Theta(T)$ and $w'_t(1,\hat{j}) \ge \varepsilon$, for sufficiently large $T$, $\lambda$ approaches zero from above.
In particular, we can strictly bound $\lambda < 0.5$.

\paragraph{Decrease in Weight for the Optimal Expert}
Next, we analyze the weight update for the optimal Expert 2.
Substituting $m_t(2) = 0.5$, we obtain
\begin{align*}
    w_t(2,j) = w'_t(2,j) \exp(\eta(j)(-0.5 + \lambda + \mu(2,j))).
\end{align*}
We consider two cases for the value of $w_t(2,j)$.

Case 1: $w_t(2,j) > \varepsilon$. Here, $\mu(2,j) = 0$.
Since $\lambda < 0.5$, the exponent $\eta(j)(-0.5 + \lambda)$ is strictly negative.
Thus, the exponential term is less than 1, implying $w_t(2,j) < w'_t(2,j)$.

Case 2: $w_t(2,j) = \varepsilon$.
Since Expert 2 is the optimal expert, its accumulated weight $w'_t(2,j)$ increases over time and satisfies $w'_t(2,j) > \varepsilon$ for all $t$.
Thus, $w_t(2,j) = \varepsilon < w'_t(2,j)$ holds.

In both cases, the optimistic update strictly decreases the weight of the optimal expert for all $j \in [M]$.

\paragraph{Mass Collapse at Large Learning Rates}
We now quantify this decrease for the largest learning rate $\hat{j}$.
Since $\eta(\hat{j}) = \Theta(T)$ and $\lambda \approx 0$, the exponent for Expert 2 is dominated by the negative term:
\begin{align*}
    \eta(\hat{j})(-0.5 + \lambda) = -\Theta(T).
\end{align*}
This large negative exponent drives the weight $w_t(2, \hat{j})$ to its lower bound $\varepsilon$.
Since $w'_1(2, \hat{j}) = \Omega(1)$ by the construction and $w'(2,j)$ increases over time,
we have $w'_t(2, \hat{j}) = \Omega(1)$.
Consequently, the mass reduction is significant:
\begin{align*}
    w'_t(2, \hat{j}) - w_t(2, \hat{j}) = \Omega(1) - \varepsilon = \Omega(1).
\end{align*}

\paragraph{Regret Analysis}
Finally, we use the conservation of total probability mass to derive the regret bound.
The sum of weights must equal 1 before and after the update:
\begin{align*}
    \sum_{j \in [M]} w_t(1,j) + \sum_{j \in [M]} w_t(2,j) = \sum_{j \in [M]} w'_t(1,j) + \sum_{j \in [M]} w'_t(2,j).
\end{align*}
Rearranging the terms, the increase in Expert 1's weight equals the decrease in Expert 2's weight:
\begin{align*}
    \sum_{j \in [M]} w_t(1,j) - \sum_{j \in [M]} w'_t(1,j) = \sum_{j \in [M]} (w'_t(2,j) - w_t(2,j)).
\end{align*}
As shown above, $w'_t(2,j) - w_t(2,j) > 0$ for all $j$, and specifically for $\hat{j}$, the difference is $\Omega(1)$.
Therefore, we have
\begin{align*}
    \sum_{j \in [M]} w_t(1,j) \ge \sum_{j \in [M]} w'_t(1,j) + \Omega(1) = \Omega(1).
\end{align*}
This implies that in every round, the algorithm assigns a constant probability mass $\Omega(1)$ to the suboptimal Expert 1 based on the misleading prediction $m_t$.
Consequently, the instantaneous regret is $\Omega(1)$, and the cumulative regret over $T$ rounds is $\Omega(T)$.

\paragraph{Contrast with \algoref{alg:proposed}}
In contrast, \algoref{alg:proposed} achieves sublinear regret by dynamically excluding unstable learning rates.
In the adversarial scenario described above, the product of the large learning rate and the prediction error is large: $32\eta(\hat{j})|l_t(1) - m_t(1)| \gg 1$.
Furthermore, the algorithm assigns a significant weight $w_t(1, \hat{j}) = \Omega(1)$ to the suboptimal expert.
These conditions trigger the penalty update mechanism, causing the cumulative penalty $L_t(\hat{j})$ to increase rapidly.
Once $L_t(\hat{j})$ exceeds the threshold, the algorithm removes the index $\hat{j}$ from the active set $\mathcal{J}_t$.
With the dangerous learning rate excluded, the algorithm relies on conservative learning rates or other safe candidates.
Therefore, as established in \thmref{thm:switching_regret}, \algoref{alg:proposed} maintains a worst-case regret bound of $O(\sqrt{T})$, successfully avoiding the linear regret suffered by the safeguard-free counterpart. 
\qed

\subsection{Proof of \lemref{lem:variable_experts}}

\paragraph{Setup and Definitions}
We define the comparators used in the regret analysis.
Consider a fixed comparator vector $u \in \Delta_K$.
We restrict the support of $u$ to the intersection of the expert sets over the analysis interval $\{t_1, \dots, t_2\}$.
Specifically, if $i \notin \bigcap_{t=t_1}^{t_2} K_t$, then $u(i) = 0$.
For an arbitrary optimal learning rate index $j^*$, we define the extended comparator $\tilde{u} \in \Delta_{K \times M}$ such that $\tilde{u}(i, j) = u(i)$ if $j=j^*$ and $\tilde{u}(i, j) = 0$ otherwise.
To ensure the comparator lies within the decision set $\Omega_t$, we define the time-dependent constrained comparator $\bar{u}_t$ as a mixture of $\tilde{u}$ and the uniform distribution.
Specifically, we define $\bar{u}_t(i, j) = (1 - \gamma_t) \tilde{u}(i, j) + \varepsilon_t$ for all $i \in K_t$ and $j \in [M]$, where $\gamma_t$ is a coefficient that ensures $\bar{u}_t$ sums to 1.

\paragraph{Regret Decomposition and Approximation Error Evaluation}
We decompose the cumulative regret against the fixed comparator $u$ into two components:
the algorithmic regret against the time-dependent constrained comparator $\bar{u}_t$ and the approximation error.
Specifically, we express the regret over the interval $\{t_1, \dots, t_2\}$ as follows:
$$\sum_{t=t_1}^{t_2} \langle \ell_t, p_t - u \rangle = \sum_{t=t_1}^{t_2} \langle \ell_t, p_t - \bar{u}_t \rangle + \sum_{t=t_1}^{t_2} \langle \ell_t, \bar{u}_t - u \rangle.$$
The first term represents the regret minimized by the algorithm under the constraint $\Omega_t$.
The second term quantifies the bias introduced by projecting the unconstrained comparator $u$ onto the feasible set defined by the lower bound $\varepsilon_t$.

We bound the approximation error term $\sum_{t=t_1}^{t_2} \langle \ell_t, \bar{u}_t - u \rangle$.
By the definition of $\bar{u}_t$, the $L_1$-norm of the difference vector satisfies $||\bar{u}_t - \tilde{u}||_1 \le |K_t| M \varepsilon_t$.
Since the total weight $W_t$ scales with $\Theta(T^2)$ due to the initialization of learning rates, the threshold $\varepsilon_t = 1/(64 W_t T)$ is of the order $O(T^{-3})$.
Consequently, the bias per round is bounded by $O(T^{-2})$.
Summing this error over the horizon $T$, the cumulative approximation error remains $O(T^{-1})$.
This magnitude is negligible.
Therefore, we focus our subsequent analysis on bounding the first term, the algorithmic regret against $\bar{u}_t$.

\paragraph{Telescoping Decomposition of the Stability Term}
We analyze the algorithmic regret using \lemref{lem:OMD_bound}.
Let $\psi_t$ denote the potential function defined on the active expert set at round $t$.
Applying \lemref{lem:OMD_bound}, we have
\begin{align*}
    \langle \tilde{\ell}_t, \tilde{p}_t - \bar{u}_t \rangle
    &\le D_{\psi_t}(\bar{u}_t, p'_t) - D_{\psi_t}(\bar{u}_t, \tilde{p}'_{t+1}) + A'_t,
\end{align*}
where $A'_t = \langle \tilde{\ell}_t - \tilde{m}_t + a_t, \tilde{p}_t - \tilde{p}'_{t+1} \rangle - D_{\psi_t}(\tilde{p}'_{t+1}, \tilde{p}_t) - \langle a_t, \tilde{p}_t - \bar{u}_t \rangle$.
Summing over the interval $\{t_1, \dots, t_2\}$, we focus on the stability term:
\begin{align*}
    \sum_{t=t_1}^{t_2} \left( D_{\psi_t}(\bar{u}_t, p'_t) - D_{\psi_t}(\bar{u}_t, \tilde{p}'_{t+1}) \right).
\end{align*}
Unlike the standard setting with a fixed comparator, the time-varying nature of $\bar{u}_t$ prevents a direct telescoping cancellation.
Therefore, we decompose this sum into the initial divergence, the final divergence, and the cumulative linking terms:
\begin{align*}
    &\quad D_{\psi_{t_1}}(\bar{u}_{t_1}, p'_{t_1}) - D_{\psi_{t_2}}(\bar{u}_{t_2}, \tilde{p}'_{t_2+1}) \\
    &+ \sum_{t=t_1}^{t_2-1} \left( D_{\psi_{t+1}}(\bar{u}_{t+1}, p'_{t+1}) - D_{\psi_t}(\bar{u}_t, \tilde{p}'_{t+1}) \right).
\end{align*}
We further decompose the linking term to isolate the effects of the weight update and the comparator change. Specifically, we split the term inside the summation into two components:
\begin{align*}
    &\quad D_{\psi_{t+1}}(\bar{u}_{t+1}, p'_{t+1}) - D_{\psi_t}(\bar{u}_t, \tilde{p}'_{t+1}) \\
    &= \underbrace{\left( D_{\psi_{t+1}}(\bar{u}_{t+1}, p'_{t+1}) - D_{\psi_t}(\bar{u}_{t+1}, \tilde{p}'_{t+1}) \right)}_{\text{Distribution Shift}} \\
    &\quad + \underbrace{\left( D_{\psi_t}(\bar{u}_{t+1}, \tilde{p}'_{t+1}) - D_{\psi_t}(\bar{u}_t, \tilde{p}'_{t+1}) \right)}_{\text{Comparator Drift}}.
\end{align*}
The first component, the distribution shift, captures the change in the potential due to the update of the expert set and weights while keeping the comparator fixed at $\bar{u}_{t+1}$.
The second component, the comparator drift, quantifies the divergence change caused solely by the variation in the comparator from $\bar{u}_t$ to $\bar{u}_{t+1}$.

\paragraph{Analysis of Distribution Shift}
We analyze the distribution shift component of the linking term, specifically focusing on the logarithmic contribution arising from the weight update.
The difference in the potential function due to the distribution update is governed by the ratio of the normalization constants.
For any expert $i$ in the support of $\tilde{u}_{t+1}$, which by definition belongs to the intersection of $K_t$ and $K_{t+1}$, the update rule implies that the probability satisfies $p'_{t+1}(i, j) = (W_t / W_{t+1}) \tilde{p}'_{t+1}(i, j)$.
Substituting this relation into the Bregman divergence, the logarithmic difference becomes
\begin{align*}
    &\quad \sum_{i \in K_t \cap K_{t+1}} \sum_{j \in [M]} \frac{\bar{u}_{t+1}(i, j)}{\eta(j)} \log \left( \frac{\tilde{p}'_{t+1}(i, j)}{p'_{t+1}(i, j)} \right) \\
    &= \sum_{i, j} \frac{\bar{u}_{t+1}(i, j)}{\eta(j)} \log \left( \frac{W_{t+1}}{W_t} \right).
\end{align*}
We approximate this weighted sum by the dominant term
\begin{align*}
    \frac{1}{\eta(j^*)} \log (W_{t+1} / W_t),
\end{align*}
which corresponds to the weight concentrated on the optimal learning rate $j^*$.
The approximation error arises from the residual probability mass of $\bar{u}_{t+1}$ distributed over indices other than $j^*$, which is proportional to the threshold $\varepsilon_{t+1}$.
Given the definition $\varepsilon_t = 1/(64 W_t T)$, this error term is of the order $O(T^{-1} \log(W_{t+1}/W_t))$ per round.
Summing over the full horizon $T$, the cumulative error is bounded by $O(T^{-1} \log T)$, which is negligible relative to the main terms.

Next, we evaluate the linear term of the Bregman divergence difference to show that the cost of weight redistribution is bounded.
The linear component is expressed as the difference between the marginal sums of the distributions
\begin{align*}
    &\quad \sum_{j \in [M]} \frac{1}{\eta(j)} \left( q_{p'_{t+1}}(j) - q_{\tilde{p}'_{t+1}}(j) \right) \\
    &= \sum_{j \in [M]} \frac{1}{\eta(j)} \left( \sum_{i \in K_{t+1}} p'_{t+1}(i, j) - \sum_{i \in K_t} \tilde{p}'_{t+1}(i, j) \right).
\end{align*}
We partition the set of experts into three disjoint categories: continuing experts ($S_t = K_t \cap K_{t+1}$), newly added experts ($N_t = K_{t+1} \setminus K_t$), and removed experts ($R_t = K_t \setminus K_{t+1}$).
For continuing experts, the relationship $p'_{t+1} = (W_t / W_{t+1}) \tilde{p}'_{t+1}$ and the monotonicity assumption $W_t \le W_{t+1}$ imply that the contribution is non-positive:
$(W_t/W_{t+1} - 1) \sum_{i \in S_t} \tilde{p}'_{t+1}(i, j) \le 0$.
Similarly, the term corresponding to removed experts, $-\sum_{i \in R_t} \tilde{p}'_{t+1}(i, j)$, is strictly non-positive.
Consequently, we discard these non-positive terms and bound the linear difference solely by the contribution from newly added experts:
\begin{align*}
    \sum_{j \in [M]} \frac{1}{\eta(j)} \sum_{i \in N_t} p'_{t+1}(i, j).
\end{align*}
This remaining term represents the initial weight assigned to new experts, which amounts to $O(1/T)$ per round and accumulates to a negligible error over the horizon.

\paragraph{Analysis of Comparator Drift}
We explicitly evaluate the comparator drift component of the linking term.
This term arises because the constrained comparator $\bar{u}_t$ varies over time to satisfy the condition $\bar{u}_t \in \Omega_t$.
Specifically, we bound the difference in Bregman divergence with respect to the fixed distribution $\tilde{p}'_{t+1}$:
\begin{align*}
    \Delta_{\text{drift}} = D_{\psi_t}(\bar{u}_{t+1}, \tilde{p}'_{t+1}) - D_{\psi_t}(\bar{u}_t, \tilde{p}'_{t+1}).
\end{align*}
Expanding the definition of the potential function, we decompose this difference into the change in entropy and the change in the linear cross-entropy term:
\begin{align*}
    \Delta_{\text{drift}}
    &= \sum_{i, j} \frac{\log \tilde{p}'_{t+1}(i, j)}{\eta(j)} ( f(\bar{u}_{t+1}(i, j)) - f(\bar{u}_t(i, j)) \\
    &\quad\quad - (\bar{u}_{t+1}(i, j) - \bar{u}_t(i, j)) ),
\end{align*}
where $f(x) = x \log x$.

We bound the magnitude of this drift by analyzing the variation in $\bar{u}_t$.
The value of $\bar{u}_t$ changes primarily due to the update of the threshold $\varepsilon_t$ and the normalization coefficient $\gamma_t$.
Since $\varepsilon_t = O(T^{-3})$ and $W_t$ increases monotonically, the element-wise difference $|\bar{u}_{t+1}(i, j) - \bar{u}_t(i, j)|$ is bounded by $O(T^{-4})$.
Both the derivative of the entropy function $f'(x) = 1 + \log x$ and the logarithmic term $\log \tilde{p}'_{t+1}(i, j)$ are bounded by $O(\log T)$ within the domain defined by $\varepsilon_t$.
Consequently, even after scaling by the inverse learning rate $1/\eta(j) = O(T)$, the total drift error per round is bounded by $O(T^{-3} \log T)$.
Summing this error over the horizon $T$, the cumulative comparator drift is $O(T^{-2} \log T)$.
This term is negligible compared to the main terms.

\paragraph{Analysis of the Initial Term}
We evaluate the initial Bregman divergence term $D_{\psi_{t_1}}(\bar{u}_{t_1}, p'_{t_1})$,
distinguishing between the intermediate case ($t_1 > 1$) and the initialization case ($t_1 = 1$).
For any intermediate round $t_1 > 1$, the distribution $p'_{t_1}$ results from a projection or update in the previous round that enforces the lower bound constraint.
Specifically, $p'_{t_1}(i, j) \ge \varepsilon_{t_1}$ holds for all active experts.
Consequently, the logarithmic barrier term in the divergence is bounded by $\frac{1}{\eta(j^*)} \log (1/\varepsilon_{t_1})$.
Substituting the definition $\varepsilon_{t_1} = 1/(64 W_{t_1} T)$, this bound simplifies to $\frac{1}{\eta(j^*)} (\log W_{t_1} + \log (64T))$.
This term naturally captures the complexity of the model space at the start of the analysis interval.

For the specific case where the analysis begins at the first round ($t_1 = 1$),
the algorithm initializes weights as $w'_1(i, j) = \eta(j)^2$ without explicitly enforcing the threshold $\varepsilon_1$.
The normalized probability is given by $p'_1(i, j) = \eta(j)^2 / W_1$.
Since the constrained comparator $\bar{u}_1$ places a probability mass of approximately $1 - O(T^{-3})$ on the optimal expert pair $(i^*, j^*)$,
the divergence is dominated by the term $\frac{1}{\eta(j^*)} \log (1/p'_1(i^*, j^*))$.
Substituting the initialization values and noting that $W_1 = \Theta(T^2)$ and $\eta(j) \ge \Omega(T^{-1})$, this term evaluates to $\frac{1}{\eta(j^*)} \log (W_1 / \eta(j^*)^2)$.
This logarithmic term scales as $O(\log T)$, which confirms that the initial divergence remains bounded within the standard logarithmic order even without the explicit enforcement of $\Omega_1$ at the initialization step.

\balance
\paragraph{Summation of Error Terms and Final Bound Derivation}
We combine the results from the previous steps to derive the final regret bound.
The total regret decomposes into the stability term, the penalty term, and the approximation error.
We first address the penalty term, which arises from the interaction between the loss vector and the weight update.
Following the analysis in Lemma 1, we bound the interaction term $A'_t$ associated with the OOMD update.

We analyze the interaction term $A'_t$ by considering two distinct regimes based on the magnitude of the learning rate and the prediction error.
Recall that $r_t(i) = \ell_t(i) - m_t(i)$ denote the prediction error.
For the regime where the learning rate is small, specifically when $32\eta(j)|r_t(i)| \le 1$,
we utilize the inequality $e^x - 1 - x \le x^2$ for $x \ge -1$.
In this case, the interaction term is bounded by a quadratic term of the prediction error, $32\eta(j) w_t(i, j) r_t(i)^2$.
This term captures the standard regret bound for bounded losses.
Conversely, for the large learning rate regime where $32\eta(j)|r_t(i)| > 1$, we cannot rely on the quadratic bound.
Instead, we bound the interaction term by the penalty $\Delta L_t(j)$, which explicitly measures the deviation caused by aggressive updates.
Summing over all experts and learning rates, the cumulative interaction term satisfies:
\begin{align*}
    &\sum_{t=t_1}^{t_2} A'_t \\
    &\le 32\eta(j^*) \sum_{t=t_1}^{t_2} \sum_{i \in K_t} \bar{u}_t(i, j^*) r_t(i)^2 + \sum_{t=t_1}^{t_2} \sum_{j \in [M]} \Delta L_t(j) + O(T^{-1}).
\end{align*}
The $O(T^{-1})$ term accounts for the residual errors from the correction term approximation.

We now sum the stability term, the penalty term, and the approximation errors derived in the preceding steps.
The stability term, including the initial divergence and the costs associated with distribution shift, is bounded by $\frac{1}{\eta(j^*)} \log (W_{t_2+1} M T)$.
The comparator drift and bias errors contribute an additional $O(T^{-2} \log T)$ term, which vanishes asymptotically.
Combining these components yields the final regret bound for the interval $\{t_1, t_2\}$.
\qed

\end{document}